\documentclass{article}

\usepackage{microtype}
\usepackage{graphicx}
\usepackage{booktabs}

\usepackage{hyperref}

\usepackage[accepted]{theconference2024}

\usepackage{amsmath}
\usepackage{amssymb}
\usepackage{mathtools}
\usepackage{amsthm}

\usepackage[capitalize,noabbrev]{cleveref}

\theoremstyle{plain}
\newtheorem{theorem}{Theorem}[section]

\newtheorem{lemma}[theorem]{Lemma}

\theoremstyle{definition}

\newtheorem{assumption}[theorem]{Assumption}
\theoremstyle{remark}
\newtheorem{remark}[theorem]{Remark}

\usepackage{amsmath}
\DeclareMathOperator*{\argmax}{arg\,max}
\DeclareMathOperator*{\argmin}{arg\,min}

\usepackage[textsize=tiny]{todonotes}

\usepackage{bm}
\theconferencetitlerunning{Arxiv Preprint}
\usepackage{enumitem}

\usepackage{subfigure}
\usepackage{caption}
\usepackage{subcaption}
\begin{document}

\twocolumn[
\theconferencetitle{Gradient Multi-Normalization for Stateless and Scalable LLM Training}
\theconferencesetsymbol{equal}{*}

\begin{theconferenceauthorlist}
\theconferenceauthor{Meyer Scetbon}{equal,msr}
\theconferenceauthor{Chao Ma}{equal,msr}
\theconferenceauthor{Wenbo Gong}{equal,msr}
\theconferenceauthor{Edward Meeds}{msr}

\end{theconferenceauthorlist}

\theconferenceaffiliation{msr}{Microsoft Research}

\theconferencecorrespondingauthor{Meyer Scetbon}{t-mscetbon@microsoft.com}
\theconferencecorrespondingauthor{Chao Ma}{chaoma@microsoft.com}
\theconferencecorrespondingauthor{Wenbo Gong}{wenbogong@microsoft.com}

\theconferencekeywords{Machine Learning, theconference}

\vskip 0.3in
]






\printAffiliationsAndNotice{\theconferenceEqualContribution}

\begin{abstract}
Training large language models (LLMs) typically relies on adaptive optimizers like Adam~\citep{adam}, which store additional state information to accelerate convergence but incur significant memory overhead. Recent efforts, such as SWAN~\citep{ma2024swansgdnormalizationwhitening}, address this by eliminating the need for optimizer states while achieving performance comparable to Adam via a multi-step preprocessing procedure applied to instantaneous gradients. Motivated by the success of SWAN, we introduce a novel framework for designing stateless optimizers that normalizes stochastic gradients according to multiple norms. To achieve this, we propose a simple alternating scheme to enforce the normalization of gradients w.r.t these norms. We show that our procedure can produce, up to an arbitrary precision, a fixed-point of the problem, and that SWAN is a particular instance of our approach with carefully chosen norms, providing a deeper understanding of its design. However, SWAN’s computationally expensive whitening/orthogonalization step limit its practicality for large LMs. Using our principled perspective,  we develop of a more efficient, scalable, and practical stateless optimizer. Our algorithm relaxes the properties of SWAN, significantly reducing its computational cost while retaining its memory efficiency, making it applicable to training large-scale models. Experiments on pre-training LLaMA models with up to 1 billion parameters demonstrate a $3$× speedup over Adam with significantly reduced memory requirements, outperforming other memory-efficient baselines.
\end{abstract}

\section{Introduction}

The training of Large Language Models (LLMs) relies heavily  on adaptive optimization algorithms, such as Adam~\cite{adam}, which dynamically adjust learning rates for each parameter based on past gradient information, leading to faster convergence and improved stability. However, these optimizers introduce substantial memory overhead due to the storage of internal states, typically moment estimates of gradients, a challenge that becomes particularly pronounced in distributed training settings where memory constraints and communication overhead are critical concerns~\cite{rajbhandari2020zero, korthikanti2023reducing, llama3}. In contrast, simpler first-order optimization methods such as Stochastic Gradient Descent (SGD) require significantly less memory but fail to adequately train LLMs~\cite{zhao2024deconstructing, zhang2020adaptive, kunstner2023noise, kunstner2024heavy}. As a result, there is an ongoing need for developing new optimization strategies that resolves the memory efficiency v.s. training performance dilemma for large-scale models training.

Recent research has made significant strides in improving the efficiency of optimization methods by reducing the memory overhead associated with saving optimizer states~\cite{hu2021lora,Lialin2023ReLoRAHT, Zhao2024GaLoreML, Hao2024FloraLA, xu2024adamlearningratescaling, jordan2024muon, zhang2024adam, ma2024swansgdnormalizationwhitening, zhu2024apollo}. Among these advancements,  ~\citet{ma2024swansgdnormalizationwhitening} introduce SWAN, a stateless optimizer that only performs pre-processing operations on the instantaneous gradients, achieving the same memory footprint as SGD while delivering comparable or even better performances than Adam. Collectively, these advances demonstrate that memory efficiency and loss throughput are not mutually exclusive, opening pathways for efficient optimization in large-scale deep learning.


\captionsetup[subfigure]{labelformat=empty}

\begin{figure*}[t]
    \centering
    \begin{minipage}{0.33\textwidth}
        \centering
        \includegraphics[width=\textwidth]{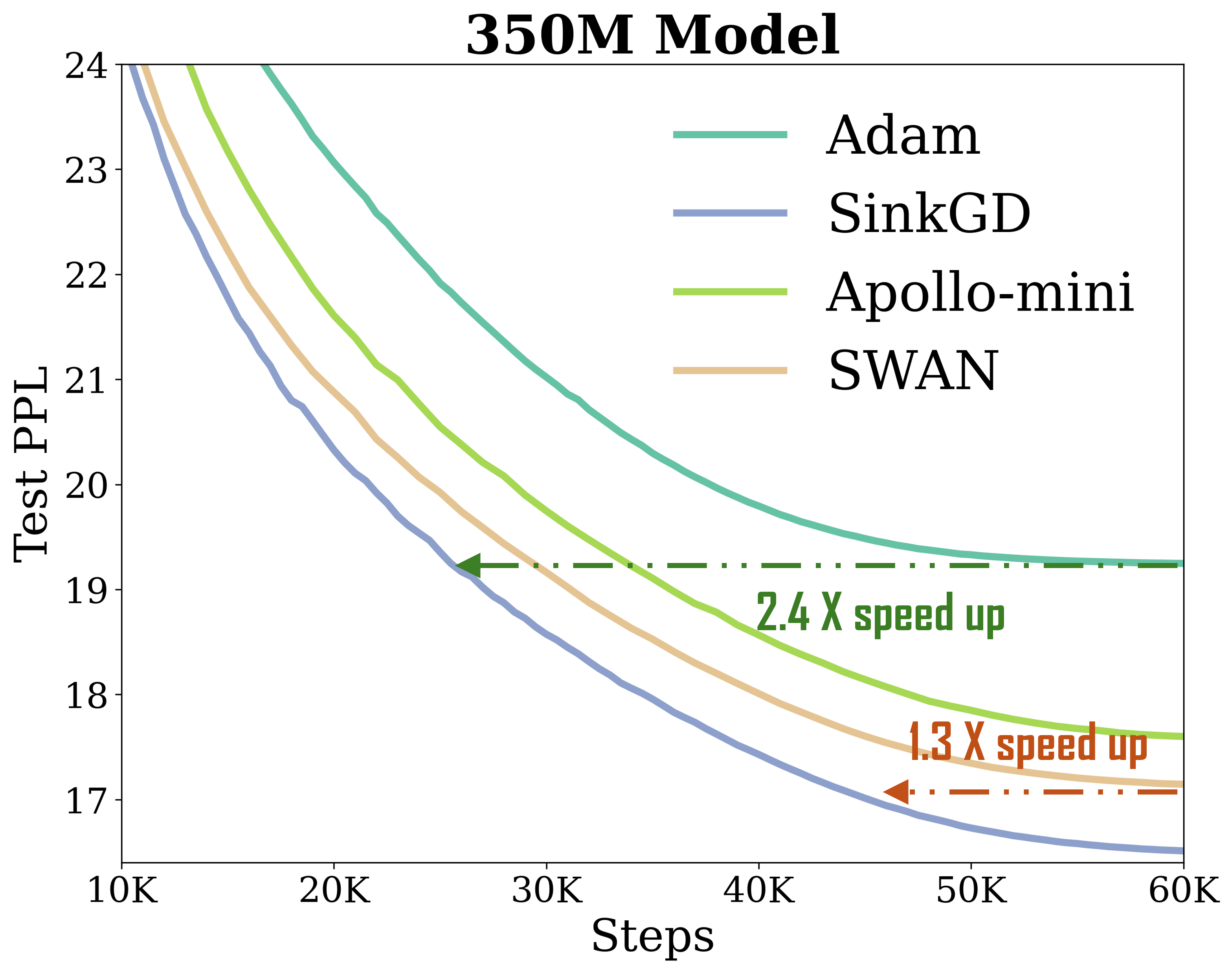}
        \caption*{(a) 350M LLaMA model}
    \end{minipage}\hfill
    \begin{minipage}{0.33\textwidth}
        \centering
        \includegraphics[width=\textwidth]{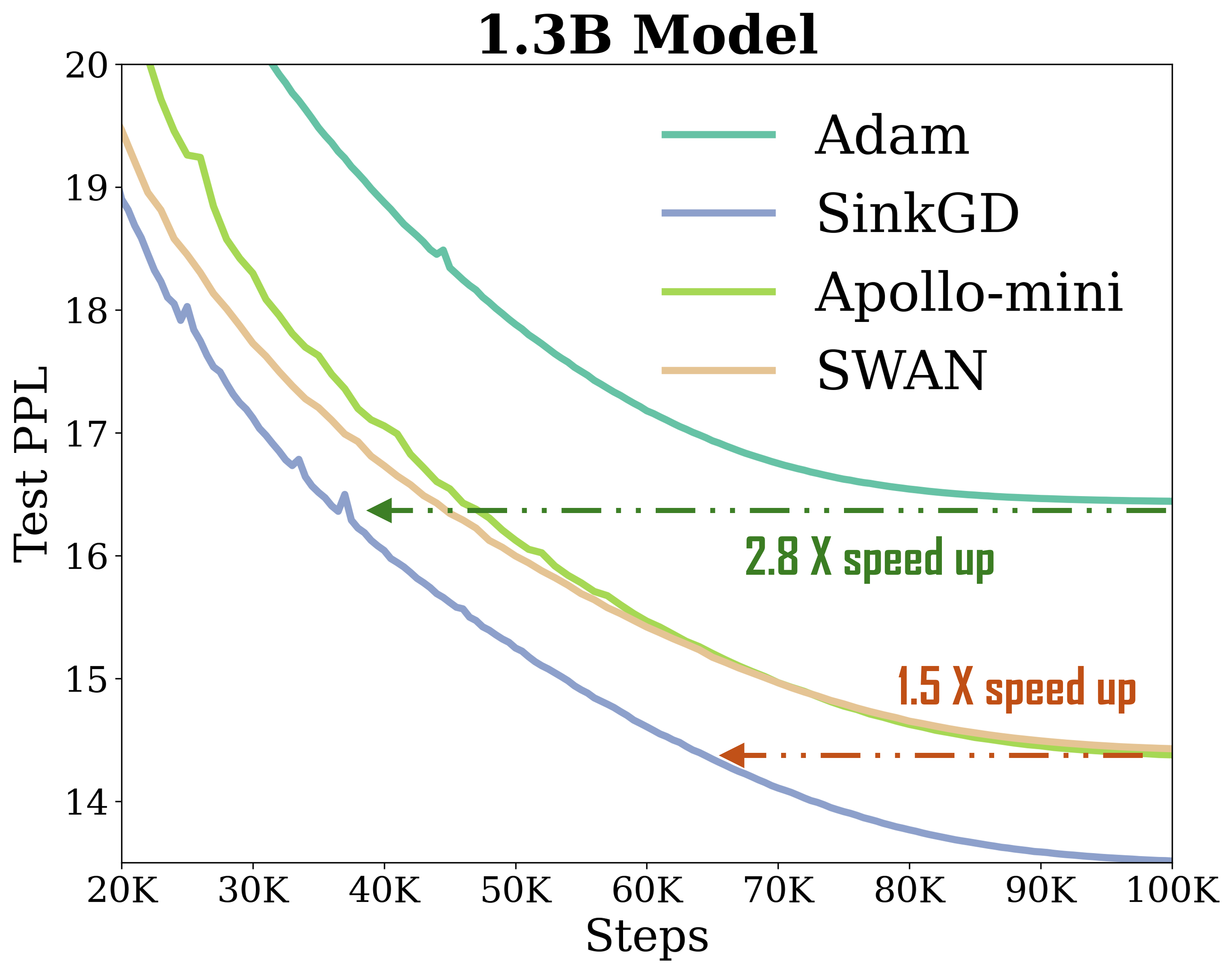}
        \caption*{(b) 1.3B LLaMA model}
    \end{minipage}\hfill
    \centering
    \begin{minipage}{0.33\textwidth}
        \centering
        \includegraphics[width=\textwidth]{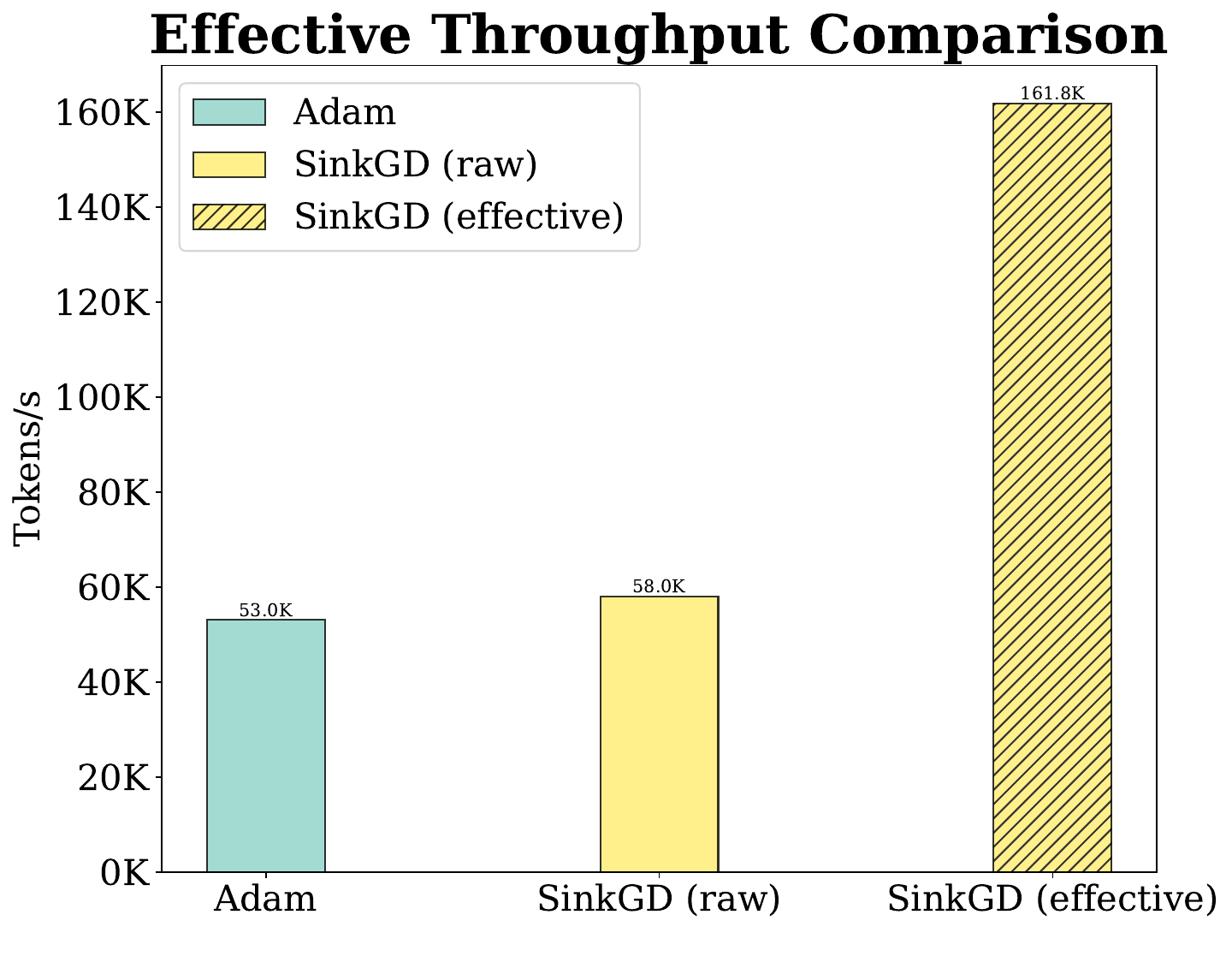}
        \caption*{(c) Training Throughputs}
    \end{minipage}\hfill
\caption{\textbf{\texttt{SinkGD} performance preview on LLM pretraining.}  \textbf{(a)} and \textbf{(b)}: Comparison of the test perplexities obtained by Adam~\cite{adam, Zhao2024GaLoreML}, SWAN \citep{ma2024swansgdnormalizationwhitening}, Apollo \citep{zhu2024apollo},  and our proposed \texttt{SinkGD} (Algorithm~\ref{alg:multi-normalized-sinkhorn}) on 1B LLaMA pretraining task with C4 dataset. All loss curves of Adam and Apollo-mini are reproduced from the corresponding opensource codes. We also compare with their official results in \Cref{tab:main}. On both 350M and 1.3B LLama architectures, \texttt{SinkGD}  achieves $>$ 2× speed-up vs Adam in terms of tokens seen; and 1.3 to 1.5 X speed-up vs SWAN and Apollo. \textbf{(c)}: Training throughput analysis on training 1.3 B model on 8 $\times$ A100, under constant batch size = 130K tokens. We present two metrics: raw throughput, measured by number of training tokens consumed per second; and effective throughput, which is raw throughput adjusted by the token efficiency of optimizer relative to Adam. \texttt{SinkGD} has a raw throughput that is marginally higher than Adam, while improving the effective throughput by $>$ 3×.
}
\label{fig: preview}
\end{figure*}

\paragraph{Contributions.} Motivated by the recent success of SWAN~\cite{ma2024swansgdnormalizationwhitening}, we introduce a framework for designing stateless optimizers based on a novel multi-normalization scheme. Unlike standard first-order methods that can be interpreted as gradient normalization according to a single norm ~\cite{bernstein2024old}, our approach aims at normalizing gradients according to multiple norms. We demonstrate that SWAN is a specific instance of our general framework. However, a key limitation of SWAN is its computational overhead: it relies on whitening/orthogonalization operation which has complexity $\mathcal{O}(m^2(m+n))$. This may hinder its scalability to large-scale training. To overcome this, we propose a new stateless optimizer that achieves Adam-level computational cost ($\mathcal{O}(mn)$), while having the same memory footprint as SGD. Moreover, it achieves on par or even outperforms Adam in LLM pretraining tasks, as well as various existing memory-efficient baselines under LLaMA architecture. Our contributions are summarized below:
\begin{itemize}[leftmargin=1em]
    \item \textbf{Multi-Normalized Gradient Descent.} 
    \begin{enumerate}[leftmargin=1em]
         \item We propose a novel family of first-order  methods, called \emph{Multi-Normalized Gradient Descent} (MNGD), that aims at normalizing gradients according to multiple norms. Our framework generalizes the steepest descent viewpoint of~\cite{bernstein2024old} that recasts popular first-order optimizers as normalization of gradients under a single norm.
    
        \item We then propose a simple alternating scheme in order to effectively compute the multi-normalization of gradients, and show that our algorithm can provide a fixed-point solution up to an arbitrary precision, ensuring the normalization of its output with respect to the norms considered. 
        
        \item We demonstrate that SWAN~\cite{ma2024swansgdnormalizationwhitening} is a particular instance of MNGD where the gradient is normalized according to two well-chosen norms: (1) the row-wise $\ell_2$-norm, and (2) the spectral norm. 
    \end{enumerate}

    \item \textbf{\texttt{SinkGD}: an efficient, scalable, and stateless optimizer.}

    \begin{enumerate}[leftmargin=1em]
         \item We leverage our framework and design a new stateless optimizer that relaxes the constraint of SWAN to improve computational efficiency. Our algorithm, namely \texttt{SinkGD}~(Algorithm~\ref{alg:multi-normalized-sinkhorn}), alternatively performs row-wise and column-wise normalization according to the Euclidean geometry. We show that \texttt{SinkGD} exactly recovers the square-root iterates of the Sinkhorn algorithm~\cite{sinkhorn1964relationship}.

    \item Finally, we evaluate our Sinkhorn-based stateless optimizer \texttt{SinkGD} by training LlaMA models on various scales, from 60m to 1.3B. Results (\Cref{fig: preview}) show that \texttt{SinkGD} manages to be on par or even outperforms the Adam optimizer, as well as other memory-efficient baselines, achieving a $3$× speedup over Adam at 1B scale, with significantly reduced memory requirements.
    \end{enumerate}

\end{itemize}

\subsection{Related Work}

\paragraph{Gradient Normalization.} Gradient normalization has emerged as a key technique in optimization, complementing its well-established role in forward-pass operations such as Layer Normalization (LayerNorm) \citep{Ba2016LayerN}. LARS and LAMB \citep{you2017lars, you2019lamb} employ global normalization to raw gradients and Adam's layer-wise updates, respectively, improving convergence and mitigating gradient pathologies in large-batch training. Apollo \citep{zhu2024apollo} introduces a channel-wise scaling approach, while SWAN \citep{ma2024swansgdnormalizationwhitening} replaces Adam's first-moment estimate with normalized gradients to stabilize gradient distributions. Theoretical analyses further underscore the importance of gradient normalization. \citet{hazan2015beyond} study its convergence properties in SGD, while \citet{cutkosky2020momentum} demonstrate that incorporating momentum enhances convergence without requiring large batches. \citet{bernstein2024old} interpret normalization in certain optimizers as a form of steepest descent under a specific norm, with SignSGD \citep{bernstein2018signsgd}, or standard gradient descent, serving as examples of gradient normalization.

\paragraph{Memory Efficient Optimizers.} Optimizers for large-scale training can reduce memory consumption primarily through two approaches: (1) low-rank approximation and (2) elimination of internal state dependencies. Low-rank optimizers project gradients onto a reduced subspace, allowing internal state updates within this subspace. ReLoRA \citep{Lialin2023ReLoRAHT} periodically merges LoRA \citep{hu2021lora} weights to restore full-rank representations. FLoRA \citep{Hao2024FloraLA} employs random Gaussian projections, whereas GaLore \citep{Zhao2024GaLoreML} utilizes singular value decomposition (SVD) for structured projections, further improved by Fira \citep{chen2024fira} via a compensation term. Apollo \citep{zhu2024apollo} minimizes memory overhead using rank-1 state representations. An alternative approach eliminates the need for internal states altogether. SWAN \citep{ma2024swansgdnormalizationwhitening} removes Adam's first and second moments through gradient normalization and whitening. Adam-mini \citep{zhang2024adam} reduces memory by leveraging block-wise second moment estimation. SGD-SaI \citep{xu2024no} obviates Adam’s second moment by precomputing learning rate scaling. Sign-based optimization \citep{chen2024symbolic} enables large-scale training using only first-moment updates. Muon \citep{jordan2024muon}, a simplification of Shampoo \citep{gupta2018shampoopreconditionedstochastictensor}, accelerates large model training via whitened first-moment updates, further demonstrating the viability of reduced-memory optimizers.

\paragraph{Alternating Projection.} Many iterative fixed-point algorithms employ alternating updates to enforce constraints or refine estimates. A classical example is the Von Neumann algorithm~\cite{von1950functional}, which alternates projections onto affine subspaces and converges to their intersection. The Sinkhorn algorithm~\cite{sinkhorn1967concerning} similarly alternates row and column normalizations, which can be seen as Bregman projections~\cite{benamou2015iterative} onto affine spaces, to approximate entropy-regularized optimal transport. While effective in Hilbert spaces, these algorithms do not generalize to arbitrary convex sets. Dykstra’s algorithm~\cite{dykstra1983algorithm} extends these methods by introducing correction terms, ensuring convergence to the exact projection. More generally, alternating projection methods have been extended through Pierra’s product space reformulation~\cite{pierra1984decomposition}, as well as modern techniques like ADMM~\cite{boyd2011distributed} and block-coordinate methods~\cite{tibshirani2017dykstra} in large-scale optimization. Despite these theoretical advances, extending alternating projection methods to non-convex settings remains a significant challenge. Recent progress includes manifold-based projection methods~\cite{ lewis2008alternating}, and proximal alternating techniques~\cite{bolte2014proximal}, which aim to improve convergence in non-convex problems, yet a comprehensive theory for convergence remains an open question.

\section{Background}

\subsection{From Adam to Stateless Optimizers}

\paragraph{Adam Optimizer.} Adam~\citep{adam} relies on accumulating internal states throughout training in order to improve the convergence. More formally, given a loss function $(\theta,x)\in\Theta\times\mathcal{X}\to \mathcal{L}(\theta,x)\in\mathbb{R}$, where $\Theta\subset\mathbb{R}^d$ is the set of learnable parameters and $\mathcal{X}$ is the set where the data resides, Adam aims at minimizing $\theta\to\mathbb{E}_{x\sim \mathbb{P}_x}(\mathcal{L}(\theta,x))$ where $\mathbb{P}_x$ is the distribution of data on $\mathcal{X}$. To achieve this, Adam computes at every step $t\geq 1$ a stochastic gradient associated with a mini-batch of input data $x^{(t)}$, and performs the following updates:
\begin{align*}
    \nabla_t &= \nabla_\theta \mathcal{L}(\theta_t,  x^{(t)})\\
    \text{m}_t & = \beta_1 \text{m}_{t-1} + (1-\beta_1) \nabla_t,\quad \hat{\text{m}}_t=\frac{\text{m}_t}{1-\beta_1^{t}} \\ 
    \text{s}_t & =  \beta_2 \text{s}_{t-1} + (1-\beta_2) \nabla_t^{\odot 2},\quad \hat{\text{s}}_t=\frac{\text{s}_t}{1-\beta_2^{t}}\\ 
    \theta_{t+1} &= \theta_t - \eta_t \frac{\hat{\text{m}}_t}{\sqrt{\hat{\text{s}}_t} +\varepsilon} 
\end{align*}
where $\odot$ is the Hadamard product, $\eta_t>0$ are global step-sizes, and $\beta_1, \beta_2>0$ are the weights of the exponential moving averages (EMAs) for the first and second moments respectively. During training, Adam optimizer stores two additional states $(\text{m}_t,\text{s}_t)$, effectively tripling the memory required to train the model compared to a simple stochastic gradient descent (SGD) scheme.

\paragraph{SWAN: a Stateless Optimizer.} Recently,~\citet{ma2024swansgdnormalizationwhitening} propose to move away from the paradigm of keeping track of internal states during the training of LLMs, and propose SWAN, a stateless optimizer that only pre-processes the stochastic gradients before updating the parameters. More precisely, they propose
to update the learnable weight matrices involved in the model using two matrix operators. Given a weight matrix $W\in\mathbb{R}^{m\times n}$, with $m\leq n$, at time $t\geq 1$, the SWAN update is:
\begin{equation}
\begin{aligned}
\label{eq:swan-update}
    \nabla_t &= \nabla_W \mathcal{L}(W_t,  x^{(t)})\\
    \tilde{\nabla}_t &=\sqrt{n} Q(\nabla_t)^{-1}\nabla_t\\
    \hat{\nabla}_t &=  \sqrt{n}(\tilde{\nabla}_t\tilde{\nabla}_t^\top)^{-1/2}\tilde{\nabla}_t\\
    W_{t+1}&= W_t - \eta_t  \hat{\nabla}_t\; , 
\end{aligned}
\end{equation}

where for a matrix $W\in\mathbb{R}^{m\times n}$, $Q(W):=\text{Diag}(\Vert W_{1,:}\Vert_2, \dots, \Vert W_{m,:}\Vert_2)$ is the diagonal matrix of size $m$ where the diagonal coefficients are the $\ell_2$-norm of the rows of $W$. To compute $(\hat{\nabla}_t\hat{\nabla}_t^\top)^{-1/2}$, the authors leverage the Newton-Schulz algorithm~\citep{song2022fast, li2018towards, huang2019iterative} instead of computing the SVD. While this approach does not require storing any additional states, it still suffers from a computational burden due to the $\mathcal{O}( m^2(n+m))$ computation of $(\nabla_t\nabla_t^\top)^{-1/2}\nabla_t$ which may limit usage for training large models.

\subsection{Steepest Descent as Gradient Normalization}
\label{sec:single-norm}
\citet{bernstein2024old} interpret several gradient descent schemes as steepest descent methods under specific norms. More formally, they propose to minimize a local quadratic model of the loss $\mathcal{L}(\cdot,x^{(t)})$ at $\theta_t$ w.r.t to a given norm $\Vert \cdot \Vert$, that is:
\begin{align*}
   \mathcal{Q}_{\Vert\cdot\Vert}(z):=\mathcal{L}(\theta_t, x^{(t)}) +\langle \nabla_t, z\rangle + \frac{\lambda_t}{2} \Vert z \Vert^2 
\end{align*}
where $\lambda_t>0$ are the sharpness parameters and $\nabla_t:=\nabla_{\theta} \mathcal{L}(\theta_t, x^{(t)})$ is the current stochastic gradient.

As shown in~\citep{bernstein2024old}, finding a minimizer of $\mathcal{Q}_{\Vert\cdot\Vert}$ can be equivalently formulated as solving:
\begin{align}
\label{eq-steepest}
  -\frac{\Vert \nabla_t\Vert_{*}}{\lambda_t} \argmax_{z\in\mathbb{R}^d:~\Vert z\Vert= 1} \langle \nabla_t, z\rangle 
\end{align}
where $\Vert x \Vert_{*}:=\sup_{z\in\mathbb{R}^d:~\Vert z\Vert= 1} \langle x, z\rangle$ is the dual norm of $\Vert x\Vert$. Their framework encompasses a large family of optimizers that perform the following update: 
\begin{equation}
\begin{aligned}
\label{eq:single-norm}
    \nabla_t &= \nabla_{\theta} \mathcal{L}(\theta_t, x^{(t)}) \\
\hat{\nabla}_t &=\argmax_{z\in\mathbb{R}^d:~\Vert z\Vert= 1} \langle \nabla_t, z\rangle\\
\theta_{t+1} &= \theta_t -\frac{\Vert \nabla_t\Vert_{*}}{\lambda_t}\hat{\nabla}_t
\end{aligned}
\end{equation}

Several popular gradient-descent schemes can be recovered using the above approach. For example, when the $\ell_2$-norm is used, one recovers standard gradient descent, while the $\ell_{\infty}$ leads to signed gradient descent~\citep{carlson2015stochastic}. However, this framework considers only  a single norm for pre-processing the raw gradient $\nabla_t$. In the following, we extend this approach to incorporate multiple norms for gradient pre-processing, enabling the design of efficient and stateless optimizers for LLM training.

\section{Multi-Normalized Gradient Descent} \label{sec: mngd}
Before presenting our approach, let us first introduce some clarifying notations.

\paragraph{Notations.} For a vector $x\in\mathbb{R}^d$, we call its normalized projection w.r.t to a given norm $\Vert \cdot\Vert$, the solution to the following optimization problem:
\begin{align}
\label{eq-single-proj}
\mathcal{P}_{\Vert \cdot \Vert}(x):=\argmax_{z:~\Vert z \Vert = 1} \langle x, z\rangle    
\end{align}
We also extend the definition of this notation if $x\in\mathbb{R}^{m\times n}$ is a matrix and $\Vert \cdot \Vert$ is a matrix norm.

\subsection{Gradient Multi-Normalization}

Let us now consider a finite family of $K\geq 1$ norms $(g_1,\dots,g_K)$. In order to pre-process the gradient $\nabla$ jointly according to these norms, we propose to consider the following optimization problem:
\begin{align}
\label{eq:multi-norm-opt}
 \argmax_{z} \langle \nabla, z\rangle~ \text{s.t.}~\forall~i\in [|1,K|],~g_i(z)=1\; .
\end{align}
Assuming the constraint set is non-empty, the existence of a maximum is guaranteed. However, this problem is NP-hard and non-convex due to the constraints, making it hard to solve efficiently for the general case of arbitrary norms.

\begin{algorithm}[!t]
   \caption{$\texttt{MultiNorm}(\nabla,L, \bm{g})$}
   \label{alg:alt-proj}
\begin{algorithmic}
   \STATE {\bfseries Input:} the stochastic gradient $\nabla_\theta\mathcal{L}(\theta_t,x^{(t)})$, the norms $\bm{g}:=(g_1,\dots,g_K)$, and $L\geq 1$ the number of iterations.
   \STATE Initialize $x=\nabla_\theta\mathcal{L}(\theta_t,x^{(t)})$.
   \FOR{$\ell=1$  {\bfseries to} $L$}
   \FOR{$i=1$ {\bfseries to} $K$}
   \STATE $x\gets \mathcal{P}_{g_i}(x):=\argmax\limits_{z:~g_i(z) = 1} \langle x, z\rangle $
   \ENDFOR
    \ENDFOR
    \STATE Return $x$
\end{algorithmic}
\end{algorithm}

\begin{remark}
Observe that when $K=1$, the problem~\eqref{eq:multi-norm-opt} recovers exactly the single normalization step used in~\cite{bernstein2024old}, as presented in~\eqref{eq:single-norm}.
\end{remark}

\begin{remark}
The convex relaxation of~\eqref{eq:multi-norm-opt}, defined as  
\begin{align}
\label{eq:multi-norm-opt-convex}
 \argmax_{z} \langle \nabla, z\rangle\quad \text{s.t.}~\forall~~i\in [|1,K|],~~g_i(z)\leq 1
\end{align}
is in fact equivalent to the single normalization case discussed in Section~\ref{sec:single-norm}, where the norm considered is $\Vert x\Vert:=\max\limits_{i\in[|1,K]} g_i(x)$. Thus, solving~\eqref{eq:multi-norm-opt-convex} is equivalent to computing the projection $\mathcal{P}_{\Vert \cdot \Vert}(\nabla)$. In Appendix~\ref{sec:convex-relaxation}, we provide a general approach to compute it using the so-called Chambolle-Pock algorithm~\cite{chambolle2011first}.
\end{remark}

While solving~\eqref{eq:multi-norm-opt} exactly might not be practically feasible in general, we propose a simple alternating projection scheme, 
presented in Algorithm~\ref{alg:alt-proj}. Notably, our method assumes that the projections $\mathcal{P}_{g_i}(\cdot)$  can be efficiently computed for all $i\in[|1,K|]$. Fortunately, when the $g_i$'s correspond to $\ell_p$-norms with $p\in[|1,+\infty|]$, or Schatten $p$-norms for matrices, closed-form solutions for these projections exist. See Appendix~\ref{sec:convex-relaxation} for more details.

\paragraph{SWAN: an Instance of $\texttt{MultiNorm}$.}  SWAN~\cite{ma2024swansgdnormalizationwhitening} applies two specific pre-processing steps to the raw gradients in order to update the weight matrices. In fact, each of these pre-processing steps can be seen as normalized projections with respect to a specific norm. More precisely, for $W\in\mathbb{R}^{m\times n}$ and $m\leq n$, let us define
\begin{align*}
g_1(W):=\frac{\max\limits_{i\in[|1,m|]} \Vert W_{i,:}\Vert_2}{\sqrt{n}}\; ,~ \text{and}~~ 
g_2(W):=\frac{\Vert W\Vert_{\sigma, \infty}}{\sqrt{n}}\; .
\end{align*}
where for $p\in [1,+\infty]$, $\Vert W\Vert_{\sigma,p}$ is the Schatten $p$-norm of $W$. Simple derivations leads to the following equalities:
\begin{align*}
    \mathcal{P}_{g_1}(W)&= \sqrt{n} Q(W)^{-1}W\\
    \mathcal{P}_{g_2}(W)&=\sqrt{n}(WW^\top)^{-1/2}W
\end{align*}
Therefore applying a single iteration ($L=1$) of Algorithm~\ref{alg:alt-proj} with norms $g_1$ and $g_2$ as defined above on the raw gradient $\nabla_t$ exactly leads to the SWAN update (Eq.~\eqref{eq:swan-update}).

\subsection{On the Convergence of \texttt{MultiNorm}}
We aim now at providing some theoretical guarantees on the convergence of  $\texttt{MultiNorm}$ (Algorithm~\ref{alg:alt-proj}). More precisely, following the SWAN implementation~\cite{ma2024swansgdnormalizationwhitening}, we focus on the specific case where $K=2$ and the normalized projections associated with the norms $g_1$ and $g_2$ have constant $\ell_2$-norm. More formally, we consider the following assumption.
\begin{assumption}
\label{assump-norm}
Let $g$ be a norm on $\mathbb{R}^d$. We say that it satisfies the assumption if for all $x\in\mathbb{R}^d$, $\Vert \mathcal{P}_{g}(x) \Vert_2 = c $ where $c>0$ is an arbitrary positive constant independent of $x$ and $\Vert\cdot\Vert_2$ represents the Euclidean norm.
\end{assumption}

\begin{remark}
Observe that both norms in SWAN satisfies Assumption~\ref{assump-norm} and their normalized projections have the same $\ell_2$-norm, as for any $W\in\mathbb{R}^{m\times n}$ with $m\leq n$, we have $\Vert \mathcal{P}_{g_1}(W) \Vert_2 = \Vert \mathcal{P}_{g_2}(W)\Vert_2 = \sqrt{nm}$.
\end{remark}

This assumption enables to obtain useful properties on $\mathcal{P}_{g}$ as we show in the following Lemma:
\begin{lemma}
\label{lem:properties-proj}
Let $g$ a norm satisfying Assumption~\ref{assump-norm}. Then
\begin{align*}
    \mathcal{P}_{g}\circ\mathcal{P}_{g} =\mathcal{P}_{g}
\end{align*}
and for all $x\in\mathbb{R}^d$, $g^*(\mathcal{P}_g(x))=\Vert \mathcal{P}_g(x)\Vert_2^2=c^2$, 
where $g^*$ is the dual norm associated with $g$.
\end{lemma}

Let us now introduce some additional notation to clearly state our result. Let $x_0\in\mathbb{R}^d$ and let us define for $n\geq 0$:
\begin{equation}
\begin{aligned}
\label{eq:seq}
    x_{2n+1}&:=\mathcal{P}_{g_1}(x_{2n})\\
    x_{2n+2}&:= \mathcal{P}_{g_2}(x_{2n+1})
    \end{aligned}
\end{equation}

which is exactly the sequence generated by Algorithm~\ref{alg:alt-proj} when $K=2$ and $x_0=\nabla_\theta\mathcal{L}(\theta_t,x^{(t)})$. Let us now show our main theoretical result, presented in the following Theorem.
\begin{theorem}
\label{thm:cvg}
Let $g_1$ and $g_2$ two norms on $\mathbb{R}^d$ satisfying Assumption~\ref{assump-norm} and such that their normalized projections have the same $\ell_2$ norm. Let also $(x_n)_{n_\geq 0}$ be defined as in~\eqref{eq:seq} and let us define the set of fixed-point as:
\begin{align*}
    \mathcal{F}:=\{x:~\mathcal{P}_{g_1}(x)=\mathcal{P}_{g_2}(x)=x\}
\end{align*}
Then by denoting $d(x,\mathcal{F}):=\min\limits_{z\in\mathcal{F}}\Vert x - z\Vert_2$ we have 
\begin{align*}
d(x_n,\mathcal{F}) \xrightarrow[n\to\infty]{} 0\; .
\end{align*}
\end{theorem}

This Theorem states that if $\texttt{MultiNorm}$ runs for a sufficient amount of time, then the returned point $x$ can be arbitrarily close to a fixed-point solution. While we cannot guarantee that it solves~\eqref{eq:multi-norm-opt}, we can assert that our algorithm converges to a fixed-point solution with arbitrary precision, and as a by-product produces a solution $x$ normalized w.r.t both norms $g_1$, $g_2$ (up to an arbitrary precision).

\begin{remark}
Note that in Theorem~\ref{thm:cvg} we assume that the normalized projections associated to $g_1$ and $g_2$ have the same $\ell_2$-norms. However, given two norms $g_1$ and $g_2$ satisfying Assumption~\ref{assump-norm}, i.e. such that for all $x$:
\begin{align*}
    \Vert \mathcal{P}_{g_1}(x) \Vert_2 &= c_1\\
    \Vert \mathcal{P}_{g_2}(x) \Vert_2 &= c_2
\end{align*}
for some $c_1,c_2>0$, and given a target value $a>0$, one can always rescale the norms such that their normalized projections have the same $\ell_2$ norm equal to $a$. More formally, by denoting $\tilde{g_1} = \frac{c_1}{a} g_1$ and $\tilde{g_2} = \frac{c_2}{a} g_2$, we obtain that
\begin{align*}
    \Vert \mathcal{P}_{\tilde{g}_1}(x)\Vert_2 =  \Vert \mathcal{P}_{\tilde{g}_2}(x)\Vert_2 = a .
\end{align*}
\end{remark}

\begin{remark}
It is worth noting that, for squared matrices ($m=n$), a single iteration ($L=1$) of \texttt{MultiNorm} using the norms considered in~\cite{ma2024swansgdnormalizationwhitening}, immediately converges to a fixed-point---precisely recovering SWAN.
\end{remark}

\subsection{MNGD: a New Family of Stateless Optimizers.} 

We now introduce our family of optimizers: \emph{Multi-Normalized Gradient Descents} (MNGDs) (Algorithm~\ref{alg:multi-normalized-gd}). The key distinction from the framework proposed in~\cite{bernstein2024old} is that MNGDs normalize the gradient with respect to multiple norms using the 
$\texttt{MultiNorm}$ step, whereas in~\cite{bernstein2024old}, the gradient is normalized using a single norm, as shown in~\eqref{eq:single-norm}.
\begin{algorithm}[!t]
   \caption{Multi-Normalized GD ($\texttt{MNGD}$)}
   \label{alg:multi-normalized-gd}
\begin{algorithmic}
   \STATE {\bfseries Input:} $T\geq 1$ the number of updates, $(\eta_t)_{0\leq t\leq T}$ the global step-sizes, $\mathcal{L}$ the loss to minimize, $L\geq 1$ the number of iterations for the multi-normalization, and $\bm{g}:=(g_1,\dots,g_K)$ the norms.
   \STATE Initialize $\theta_0$
   \FOR{$t=1$  {\bfseries to} $T$}
    \STATE $\nabla_t\gets \nabla_{\theta}\mathcal{L}(\theta_t, x^{(t)})$ with $x^{(t)}\sim P_x$
    \STATE $\hat{\nabla}_t \gets \texttt{MultiNorm}(\nabla_t,L, \bm{g})$ as defined in Alg.~\ref{alg:alt-proj}.
    \STATE $\theta_{t+1} \gets \theta_t - \eta_t \hat{\nabla}_t$
    \ENDFOR
    \STATE Return $x$
\end{algorithmic}
\end{algorithm}

In the following, we focus on the MNGD scheme with a specific choice of norms, for which we can efficiently compute the gradient multi-normalization step. This enables the application of stateless optimizers to large LMs.

\section{Sinkhorn: a Multi-Normalization Procedure}
As in SWAN~\cite{ma2024swansgdnormalizationwhitening}, we propose to normalize the weight matrices according to multiple norms. We still leverage the row-wise $\ell_2$-norm to pre-process raw gradients, however, rather than using the spectral norm, we propose to consider instead a relaxed form of this constraint and use the column-wise $\ell_2$-norm. More formally, let us consider the two following norms on matrices of size $\mathbb{R}^{m\times n}$:
\begin{align*}
g_1(W):=\frac{\max\limits_{i\in[|1,m|]} \Vert W_{i,:}\Vert_2}{\sqrt{n}}\; ,\quad 
g_2(W):=\frac{\max\limits_{j\in[|1,n|]} \Vert W_{:,j}\Vert_2}{\sqrt{m}}\; ,
\end{align*}
which leads to the following two normalized projections:
\begin{align*}
    \mathcal{P}_{g_1}(W)&= \sqrt{n} Q(W)^{-1}W\\
    \mathcal{P}_{g_2}(W)&=\sqrt{m}W R(W)^{-1}
\end{align*}
where $R(W):=\text{Diag}(\Vert W_{:,1}\Vert_2,\dots,\Vert W_{:,n}\Vert_2)\in\mathbb{R}^{n\times n} $ is the diagonal matrix of size $n$ with the $\ell_2$-norm of the columns of $W$ as diagonal coefficients. For such a choice of norms, the $\texttt{MultiNorm}$ reduces to a simple procedure as presented in Algorithm~\ref{alg:Sinkhorn}.

\begin{remark}
For such a choice of norms, we obtain $\Vert \mathcal{P}_{g_1}(W) \Vert_2 = \Vert \mathcal{P}_{g_2}(W)\Vert_2 = \sqrt{nm}$ for any $W\in\mathbb{R}^{m\times n}$. In other words, both norms satisfy Assumption~\ref{assump-norm} and their $\ell_2$ norms are equal to $\sqrt{nm}$.
\end{remark}

For completeness we include  the MNGD scheme (Algorithm~\ref{alg:multi-normalized-sinkhorn}) that replaces the $\texttt{MultiNorm}$ step with $\texttt{SR-Sinkhorn}$ (Algorithm~\ref{alg:Sinkhorn}).

\begin{algorithm}[!t]
   \caption{$\texttt{SR-Sinkhorn}(\nabla,L)$}
   \label{alg:Sinkhorn}
\begin{algorithmic}
   \STATE {\bfseries Input:} the stochastic gradient $\nabla_W\mathcal{L}(W_t,x^{(t)})$, and $L\geq 1$ the number of iterations.
   \STATE Initialize $X=\nabla_W\mathcal{L}(W_t,x^{(t)})\in\mathbb{R}^{m\times n}$.
   \FOR{$\ell=1$  {\bfseries to} $L$}
   \STATE $X\gets  \sqrt{n} Q(X)^{-1}X$
   \STATE $X\gets  \sqrt{m} XR(X)^{-1}$
   \ENDFOR
    \STATE Return $X$
\end{algorithmic}
\end{algorithm}

\paragraph{The Sinkhorn Algorithm.} Before explicitly showing the link between Algorithm~\ref{alg:Sinkhorn} and the Sinkhorn algorithm, let us first recall the Sinkhorn theorem~\cite{sinkhorn1964relationship} and the Sinkhorn algorithm~\cite{sinkhorn1967concerning}. Given a positive coordinate-wise matrix $A\in\mathbb{R}_{+}^{m\times n}$, there exists a unique matrix $P\in\mathbb{R}_{+}^{m\times n}$ of the form $P=QAR$ with $Q$ and $R$ positive coordinate-wise and diagonal matrices of size $m$ and $n$ respectively, such that $P\bm{1}_n=n\bm{1}_m$ and $P^\top\bm{1}_m=m\bm{1}_n$. To find $P$, one can use the Sinkhorn algorithm that initializes $P_0:=A$ and computes for $k\geq 0$:
\begin{align*}
    P_{k+1/2}&=n\text{Diag}(P_k\bm{1}_n)^{-1}P_k\\
    P_{k+1}&=m P_{k+1/2}\text{Diag}(P_{k+1/2}^\top\bm{1}_m)^{-1}\; .
\end{align*}
Equivalently, these updates on $P$ can be directly expressed as updates on the diagonal coefficients of $Q=\text{Diag}(u)$ and $R=\text{Diag}(v)$ with $u\in\mathbb{R}_{+}^m$ and $v\in\mathbb{R}_{+}^n$. By initializing $u_0=\bm{1}_m$ an $v_0=\bm{1}_m$, the above updates can be reformulated as follows:
\begin{align}
\label{eq:update-diag-sin}
    u_{k+1} = n\frac{ \bm{1}_m}{Av_k},~~v_{k+1} = m\frac{\bm{1}_n}{ A^\top u_{k+1}}
\end{align}
where $/$ denote the coordinate-wise division.~\citet{franklin1989scaling} show the linear convergence of Sinkhorn’s iterations. More formally, they show that $(u_k, v_k)$ converges to some $(u^*,v^*)$ such that $P:=\text{Diag}(u^*)A\text{Diag}(v^*)$ satisfies $P\bm{1}_n=n\bm{1}_m$ and $P^\top\bm{1}_m=m\bm{1}_n$, and:
\begin{align*}
    d_\mathcal{H}(u_k, u^*)\in\mathcal{O}(\lambda(A)^{2k})~, \text{ and } ~d_\mathcal{H}(v_k, v^*)\in\mathcal{O}(\lambda(A)^{2k})\; ,
\end{align*}
where $d_{\mathcal{H}}$ is the Hilbert projective metric~\cite{de1993hilbert} and $\lambda(A)<1$ is a contraction factor associated with the matrix $A$.

\paragraph{Links between Sinkhorn and Algorithm~\ref{alg:Sinkhorn}.} Algorithm~\ref{alg:Sinkhorn} can be seen as a simple reparameterization of the updates presented in~\eqref{eq:update-diag-sin}. More precisely, given a gradient $\nabla\in\mathbb{R}^{m\times n}$ and denoting $A:=\nabla^{\odot 2}$, we obtain that the iterations of Algorithm~\ref{alg:Sinkhorn} exactly compute:
\begin{align}
\label{eq:update-diag-sin-sr}
    u_{k+1}^{1/2} = \sqrt{n\frac{ \bm{1}_m}{Av_k}},~~v_{k+1}^{1/2} = \sqrt{m\frac{\bm{1}_n}{ A^\top u_{k+1}}}
\end{align}
where the square-root is applied coordinate-wise, and returns after $L$ iterations $X_L=\text{Diag}(u_{L}^{1/2})\nabla \text{Diag}(v_{L}^{1/2})$. Therefore the linear convergence of Algorithm~\ref{alg:Sinkhorn} follows directly from the convergence rate of Sinkhorn, and Algorithm~\ref{alg:Sinkhorn} can be thought as applying the square-root Sinkhorn algorithm, thus the name $\texttt{SR-Sinhkorn}$. Note also that at convergence ($L\to+\infty$) we obtain $X^{*}\in\mathbb{R}^{m\times n}$ which is a fixed-point of both normalized projections, that is $\mathcal{P}_{g_1}(X^*)=\mathcal{P}_{g_2}(X^*)=X^*$,
from which we deduce that
\begin{align*}
\Vert X^*_{i,:}\Vert_2 = \sqrt{n}\;, \quad \text{and}\quad   \Vert X^*_{:,j}\Vert_2 = \sqrt{m}\;
\end{align*}
as demonstrated in Theorem~\ref{thm:cvg}.

\textbf{On the Importance of the Scaling.} Now that we have shown the convergence $\texttt{SR-Sinkhorn}$, let us explain in more detail the scaling considered for both the row-wise and column-wise normalizations. First recall that both norm $g_1$ and $g_2$ satisfy Assumption~\ref{assump-norm} and that the $\ell_2$ norm of their normalized projections is equal to $\sqrt{nm}$. The reason for this specific choice of scaling ($\sqrt{nm}$) is due to the global step-size in Algorithm~\ref{alg:multi-normalized-sinkhorn}. In our proposed MNGD, we did not prescribe how to select $\eta_t$. In practice, we aim to leverage the same global step-sizes as those used in Adam~\cite{adam} for training LLMs, and therefore we need to globally rescale the (pre-processed) gradient accordingly. To achieve that, observe that when EMAs are turned-off, Adam corresponds to a simple signed gradient descent, and therefore the Frobenius norm of the pre-processed gradient is simply $\sqrt{nm}$. Thus, when normalizing either the rows or the columns, we only need to rescale the normalized gradient accordingly.

\begin{algorithm}
   \caption{Sinkhorn GD ($\texttt{SinkGD}$)}
   \label{alg:multi-normalized-sinkhorn}
\begin{algorithmic}
   \STATE {\bfseries Input:} $T\geq 1$ the number of updates, $(\eta_t)_{0\leq t\leq T}$ the global step-sizes, $\mathcal{L}$ the loss to minimize, and $L\geq 1$ the number of iterations for the SR-Sinkhorn procedure.
   \STATE Initialize $\theta_0$
   \FOR{$t=1$  {\bfseries to} $T$}
    \STATE $\nabla_t\gets \nabla_{\theta}\mathcal{L}(\theta_t, x^{(t)})$ with $x^{(t)}\sim P_x$
    \STATE $\hat{\nabla}_t \gets \texttt{SR-Sinkhorn}(\nabla_t,L)$ as defined in Alg.~\ref{alg:Sinkhorn}.
    \STATE $\theta_{t+1} \gets \theta_t - \eta_t \hat{\nabla}_t$
    \ENDFOR
    \STATE Return $x$
\end{algorithmic}
\end{algorithm}
\vspace{-0.2cm}

\textbf{Computational Efficiency of SinkGD over SWAN.} Compared to SWAN~\cite{ma2024swansgdnormalizationwhitening}, the proposed approach,  \texttt{SinkGD}, is more efficient as it only requires $\mathcal{O}(nm)$ numerical operations. In contrast, SWAN, even when implemented with Newton-Schulz, still requires performing matrix-matrix multiplications, which have a time complexity of $\mathcal{O}(m^2(m+n))$. In the next section, we will demonstrate the practical effectiveness of MNGD with $\texttt{SR-Sinkhorn}$, that is \texttt{SinkGD}. This approach manages to be on par with, and even outperforms,  memory-efficient baselines for pretraining the family of LLaMA models up to 1B scale.


\section{Experimental Results}
\label{sec:pretrain}
In this section, we evaluate the empirical performance of applying \texttt{SinkGD} optimizer to LLM pretraining tasks. All experiments were performed on NVIDIA A100 GPUs.

\subsection{LlaMA Pre-training Tasks} 

\paragraph{Setup.} We evaluate \textbf{SinkGD} on LLM pre-training tasks using a LLaMA-based architecture~\cite{touvron2023llama} with RMSNorm and SwiGLU activations \citep{zhang2019root, gao2023eigenvalue}. We consider models with 60M, 130M, 350M, and 1.3B parameters, all trained on the C4 dataset~\cite{2020t5} using an effective token batch size of 130K tokens (total batch size 512, context length 256). Specifically, for both 130M and 350M, we use 128 batch size with 4 accumulations. For 60M and 1B, we uses 256 batch with 2 accumulation, and 32 per-device batch size with 2 accumulation and 8xA100s, respectively.
Following the setup of \cite{Zhao2024GaLoreML, zhu2024apollo}, $\textbf{SinkGD}$ is applied to all linear modules in both attention and MLP blocks with $L=5$ iterations for the $\texttt{SR-Sinkhorn}$ procedure. For all other modules, that are the embedding layer, the RMSnorm layers, and the last output layer, \textbf{Adam} optimizer~\cite{adam} is used. We use exactly the same cosine learning rate scheduler as in \cite{Zhao2024GaLoreML}, where $10\%$ of total training steps is used for warm-up. Note that, as in~\cite{Zhao2024GaLoreML,zhu2024apollo}, we use a group-wise learning rate for our optimizer. The effective learning rate used for linear modules in the transformer blocks is of the form $\alpha \eta_t$ where $\eta_t$ is global learning rate provided by the scheduler and $\alpha$ is fixed hyperparameter that we set to $\alpha=0.05$. For Adam, we use $\eta_t$ as the learning rate. 

\paragraph{Baselines.} We consider the following memory-efficient optimizers baselines: \textbf{Adam} \citep{adam};  \textbf{Galore}~\cite{Zhao2024GaLoreML}; \textbf{Fira}~\cite{chen2024firaachievefullranktraining}, \textbf{Apollo} and \textbf{Apollo-mini} \cite{zhu2024apollo}, and  \textbf{SWAN}~\cite{ma2024swansgdnormalizationwhitening}. For all methods, training uses BF16 precision for weights, gradients and optimizer states by default, except for $\textbf{SWAN}$ that uses FP32 precision to pre-process the gradient~\cite{ma2024swansgdnormalizationwhitening}. We also perform a grid search of learning rate for Adam over $\{0.01, 0.005, 0.001, 0.0005, 0.0001\}$, except for 1B model which we search over $\{ 0.001, 0.0007, 0.0005, 0.0003, 0.0001\}$. We do not perform any weight decay for all optimizers. 

\begin{table*}[t!]
    \centering
    \caption{\small{Comparison with Adam and  memory-efficient baselines on pre-training various sizes of LLaMA models with C4 dataset. Test perplexity is reported, along with a memory estimate of the total of parameters and optimizer states based on BF16 format. The perplexities reported for all competitive methods are taken from \citet{Zhao2024GaLoreML, zhu2024apollo}, as well as the \textbf{SWAN} results taken from \citet{ma2024swansgdnormalizationwhitening}. Remarks: 1) As we cannot reproduce the Adam results from \citet{Zhao2024GaLoreML}, we report both the reported Adam results from \citet{Zhao2024GaLoreML} and our reproduced result; 2) The memory estimations from \citet{Zhao2024GaLoreML, zhu2024apollo} did not consider the fact that Adam optimizer was used for embedding layers. This is corrected in our memory estimates. 3) The results of \textbf{SWAN} from \citet{ma2024swansgdnormalizationwhitening} assumes no learning rate warm-up and no learning rate tuning. For fair comparison, we also report the performance of \textbf{SWAN} that matches the setting of \textbf{Galore} and \textbf{Apollo}, where in with learning rate warm-up and larger learning rates are allowed. This is denoted by \textbf{SWAN}$^\dag$. }  }
    \label{tab:main}
    \begin{tabular}{|c|cc|cc|cc|cc|}
    \toprule
               Methods & \multicolumn{2}{c|}{\textbf{60M}} & \multicolumn{2}{c|}{\textbf{130M}} & \multicolumn{2}{c|}{\textbf{350M}} & \multicolumn{2}{c|}{\textbf{1.3B}} \\
  
    & PPL & MEM & PPL & MEM & PPL & MEM & PPL & MEM\\
     \midrule 
    Adam (reproduced) & 33.94 & 0.32G  & 25.03 &0.75G        & 19.24 & 2.05G  & 16.44 &7.48G \\
    Adam (cited) &  34.06 & 0.32G  & 25.08 &0.75G        & 18.80 & 2.05G  & 15.56 &7.48G \\
    \midrule 
     
     Galore & 34.88 &0.26G& 25.36& 0.57G &18.95& 1.29G& 15.64& 4.43G\\
    Fira & 31.06 &0.26G& \textbf{22.73} & 0.57G& 17.03& 1.29G& 14.31& 4.43G \\
    Apollo-mini &  31.93 &0.23G& 23.53& 0.43G& 17.18& 0.93G& 14.17& 2.98G \\
    Apollo &  31.55 &0.26G& 22.94& 0.57G& 16.85& 1.29G& 14.20& 4.43G \\

    SWAN &  32.28 &0.23G& 24.13 &0.43G& 18.22 &0.93G& 15.13& 2.98G \\

    SWAN$^\dag$ &  \textbf{30.00} &0.23G& 22.83 &0.43G& 17.14 &0.93G& 14.42& 2.98G \\
    
    \midrule
    SinkGD & $30.99$ & 0.23G & \textbf{22.75} &0.43G &  
    \textbf{16.51} &0.93G & \textbf{13.51}  &2.98G \\
    \bottomrule
    SinkGD speed up v.s. Adam (reproduced) & \multicolumn{2}{c|}{1.60 X} & \multicolumn{2}{c|}{1.56 X} & \multicolumn{2}{c|}{2.42 X} & \multicolumn{2}{c|}{2.79 X} \\
    SinkGD speed up v.s. Adam (cited) & \multicolumn{2}{c|}{1.66 X} & \multicolumn{2}{c|}{1.73 X} & \multicolumn{2}{c|}{2.10 X} & \multicolumn{2}{c|}{2.17 X} \\
    Total Training Steps & \multicolumn{2}{c|}{10K} & \multicolumn{2}{c|}{20K} & \multicolumn{2}{c|}{60K} & \multicolumn{2}{c|}{100K}  \\ 
    \bottomrule
    \end{tabular}
\end{table*}

\begin{table}[H]
\vspace{-0.5cm}
\centering
\caption{Comparison of the test perplexities obtained during training when training 1B LLaMA with SinkGD v.s. 7B LLaMA using different baselines. For Apollo, Apollo-mini, 8-bit Adam and Galore, we cite the number from \citet{zhu2024apollo}.}
\label{tab: 7B preformance}
\resizebox{\columnwidth}{!}{
\begin{tabular}{l|l|llll}
\hline
Method           & Mem.   & 40K & 80K & 120K & 150K \\ \hline
8-bit Adam (7B)&26G &18.09 & 15.47 & 14.83 & 14.61  \\
8-bit Galore (7B) &18G & 17.94 & 15.39 &14.95 &14.65\\
Apollo (7B)      & 15.03G &  17.55   &  14.39   &  13.23    &   13.02   \\
Apollo-mini (7B) & 13.53G &  18.03   &  14.60   &  13.32    &  13.09    \\\hline
SinkGD (1B)    & 2.98G   &  \textbf{16.44}   &  \textbf{14.27}   & \textbf{13.17}     & \textbf{12.97}    \\   

\hline
\end{tabular}}
\end{table}

\paragraph{Performance evaluation and memory efficiency analysis.} The results presented in~\Cref{tab:main} demonstrate the effectiveness of \textbf{SinkGD} in terms of both computational efficiency and model performance. Notably, \textbf{SinkGD} achieves competitive performance while maintaining the lowest estimated memory consumption, comparable to that of SGD. Across all evaluated models, our method performs on par with or even surpasses \textbf{Adam} and other memory-efficient baselines in terms of test perplexity. In particular, \textbf{SinkGD} outperforms all other baselines in this experimental setup for the 350M and 1.3B model variants. Additionally, we quantify the computational efficiency of \textbf{SinkGD} by measuring the speed-up relative to \textbf{Adam}. This is determined by computing the ratio of the total training steps of \textbf{Adam} to the number of steps needed for \textbf{SinkGD} to reach the same final test perplexity. Note also that the reported memory consumption values in~\Cref{tab:main} account for three components: (1) memory allocated for model parameters, (2) optimizer-related costs for linear modules within transformer blocks, and (3) the Adam optimizer's memory footprint for the remaining parameters.

\begin{table}[H]
\vspace{-0.5cm}
\centering
\caption{Raw and effective throughput analysis. }
\label{tab: throughput}
\begin{tabular}{l|l}
\hline
      Method   & Raw / eff. throughput \\ \hline
Adam & 53047 / 53047 (tokens/s)          \\
SinkGD     & 57982 / 161769  (tokens/s)         \\ \hline
\end{tabular}
\end{table}

\paragraph{Comparative analysis of 1B and 7B LLaMA training.} To further evaluate the efficacy of our proposed optimizer, we replicate the experimental setup of \cite{zhu2024apollo}, but instead train a 1B-parameter LLaMA model using \textbf{SinkGD} and compare its performance against a 7B-parameter LLaMA model trained with \textbf{Apollo}, \textbf{Apollo-mini}, \textbf{8-bit Adam}, and \textbf{8-bit Galore}. As shown in Table \ref{tab: 7B preformance}, the 1B model trained with \textbf{SinkGD} achieves comparable test perplexities to those of the 7B model trained with \textbf{Apollo} after 150K optimization steps, while incurring significantly lower costs. Notably, training the 7B LLaMA model with Apollo requires \textbf{15} days on an 8xA100 GPU setup to complete 150K steps, whereas our approach achieves a similar loss in \textbf{3.3} days. The reported memory estimates correspond to the total memory cost detailed in the previous paragraph.

\subsection{Ablation Study}
\label{sec:ablation}

\paragraph{Throughput analysis.} We also assess throughput when training a 1.3B-parameter model on 8xA100 GPUs. We use two metrics: (1) the \emph{raw throughput} which is the number of tokens processed per second, and (2) the \emph{effective throughput} defined as the total training token used by Adam divided by the time (in seconds) used by \textbf{SinkGD} to reach the same test perplexities.
These metrics evaluate the impact of the multi-normalization step on training speed, and also account for the fact that some optimizers make more effective use of training tokens. As shown in \Cref{tab: throughput}, \textbf{SinkGD} achieves competitive raw throughput compared to \textbf{Adam}, suggesting the multi-normalization step does not require expensive computations. Furthermore, \textbf{SinkGD} exhibits a 3 $\times$ higher effective throughput than Adam, indicating a significantly faster wall-clock time convergence.

\paragraph{On the effect of the number of iterations.} In this experiment, we measure the effect of applying different iterations of our proposed $\texttt{MultiNorm}$ (Algorithm~\ref{alg:alt-proj}) scheme in the specific case of the \textbf{SWAN} and \textbf{SinkGD} methods. More specifically, we train a 130M LLaMA model on C4 datasets and compare the test perplexities obtained after 10K steps. We observe that the number of iterations has marginal effect on the performance of the algorithms. However, as we still observe a consistent improvement when using $L=5$ iterations, we decide to use this number of iterations in our benchmark evaluation, as reported in table~\ref{tab:main}.  

\begin{table}[H]
\vspace{-0.5cm}
\centering
\caption{Comparison of the test perplexities obtained during training at 10K steps when training 130M LLaMA model with either SWAN or SinkGD using different number of iterations in \texttt{MultiNorm} procedure.}
\label{tab:num-iter}
\begin{tabular}{|c|c|}
\hline
Method           & PPL   \\ 
\hline
SWAN ($L=1$) & 26.79 \\ 
SWAN ($L=5$) & 26.56\\ 
\hline
SinkGD ($L=1$)     & 26.21  \\ 
SinkGD ($L=5$)    & 26.13  \\ 
\hline
\end{tabular}
\end{table}
\paragraph{Conclusion.}
In this work, we present a novel framework for designing stateless optimization algorithms tailored for LLM training. Our approach is based on the normalization of stochastic gradients with respect to multiple norms, and we propose an alternating optimization procedure to achieve this normalization efficiently. We establish that our multi-normalization scheme can approximate, to arbitrary precision, a fixed point of the optimization problem, thereby ensuring that the gradient is appropriately scaled according to both norms. Furthermore, we extend and improve upon SWAN~\cite{ma2024swansgdnormalizationwhitening}, a specific instance of our framework, by introducing SinGD, a stateless optimizer that enforces row-wise and column-wise $\ell_2$ normalization. We demonstrate that this procedure is theoretically guaranteed to converge and provide empirical evidence that it outperforms SoTA memory-efficient optimizers, as well as Adam, in training a 1B-parameter LLaMA model on the C4 dataset. Future research directions include exploring alternative normalization schemes to further enhance the efficiency of stateless optimizers and extending the applicability of SinGD to other training regimes beyond LLM pre-training.


\clearpage
\newpage
\bibliography{biblio}

\begin{thebibliography}{52}
\providecommand{\natexlab}[1]{#1}
\providecommand{\url}[1]{\texttt{#1}}
\expandafter\ifx\csname urlstyle\endcsname\relax
  \providecommand{\doi}[1]{doi: #1}\else
  \providecommand{\doi}{doi: \begingroup \urlstyle{rm}\Url}\fi

\bibitem[Ba et~al.(2016)Ba, Kiros, and Hinton]{Ba2016LayerN}
Ba, J., Kiros, J.~R., and Hinton, G.~E.
\newblock Layer normalization.
\newblock \emph{ArXiv}, abs/1607.06450, 2016.
\newblock URL \url{https://api.semanticscholar.org/CorpusID:8236317}.

\bibitem[Benamou et~al.(2015)Benamou, Carlier, Cuturi, Nenna, and Peyr{\'e}]{benamou2015iterative}
Benamou, J.-D., Carlier, G., Cuturi, M., Nenna, L., and Peyr{\'e}, G.
\newblock Iterative bregman projections for regularized transportation problems.
\newblock \emph{SIAM Journal on Scientific Computing}, 37\penalty0 (2):\penalty0 A1111--A1138, 2015.

\bibitem[Bernstein \& Newhouse(2024)Bernstein and Newhouse]{bernstein2024old}
Bernstein, J. and Newhouse, L.
\newblock Old optimizer, new norm: An anthology.
\newblock \emph{arXiv preprint arXiv:2409.20325}, 2024.

\bibitem[Bernstein et~al.(2018)Bernstein, Wang, Azizzadenesheli, and Anandkumar]{bernstein2018signsgd}
Bernstein, J., Wang, Y.-X., Azizzadenesheli, K., and Anandkumar, A.
\newblock signsgd: Compressed optimisation for non-convex problems.
\newblock In \emph{International Conference on Machine Learning}, pp.\  560--569. PMLR, 2018.

\bibitem[Bolte et~al.(2014)Bolte, Sabach, and Teboulle]{bolte2014proximal}
Bolte, J., Sabach, S., and Teboulle, M.
\newblock Proximal alternating linearized minimization for nonconvex and nonsmooth problems.
\newblock \emph{Mathematical Programming}, 146\penalty0 (1):\penalty0 459--494, 2014.

\bibitem[Boyd et~al.(2011)Boyd, Parikh, Chu, Peleato, Eckstein, et~al.]{boyd2011distributed}
Boyd, S., Parikh, N., Chu, E., Peleato, B., Eckstein, J., et~al.
\newblock Distributed optimization and statistical learning via the alternating direction method of multipliers.
\newblock \emph{Foundations and Trends{\textregistered} in Machine learning}, 3\penalty0 (1):\penalty0 1--122, 2011.

\bibitem[Carlson et~al.(2015)Carlson, Cevher, and Carin]{carlson2015stochastic}
Carlson, D., Cevher, V., and Carin, L.
\newblock Stochastic spectral descent for restricted boltzmann machines.
\newblock In \emph{Artificial Intelligence and Statistics}, pp.\  111--119. PMLR, 2015.

\bibitem[Chambolle \& Pock(2011)Chambolle and Pock]{chambolle2011first}
Chambolle, A. and Pock, T.
\newblock A first-order primal-dual algorithm for convex problems with applications to imaging.
\newblock \emph{Journal of mathematical imaging and vision}, 40:\penalty0 120--145, 2011.

\bibitem[Chen et~al.(2024{\natexlab{a}})Chen, Feng, Li, Lai, Yue, Yuan, and Wang]{chen2024fira}
Chen, X., Feng, K., Li, C., Lai, X., Yue, X., Yuan, Y., and Wang, G.
\newblock Fira: Can we achieve full-rank training of llms under low-rank constraint?
\newblock \emph{arXiv preprint arXiv:2410.01623}, 2024{\natexlab{a}}.

\bibitem[Chen et~al.(2024{\natexlab{b}})Chen, Feng, Li, Lai, Yue, Yuan, and Wang]{chen2024firaachievefullranktraining}
Chen, X., Feng, K., Li, C., Lai, X., Yue, X., Yuan, Y., and Wang, G.
\newblock Fira: Can we achieve full-rank training of llms under low-rank constraint?, 2024{\natexlab{b}}.
\newblock URL \url{https://arxiv.org/abs/2410.01623}.

\bibitem[Chen et~al.(2024{\natexlab{c}})Chen, Liang, Huang, Real, Wang, Pham, Dong, Luong, Hsieh, Lu, et~al.]{chen2024symbolic}
Chen, X., Liang, C., Huang, D., Real, E., Wang, K., Pham, H., Dong, X., Luong, T., Hsieh, C.-J., Lu, Y., et~al.
\newblock Symbolic discovery of optimization algorithms.
\newblock \emph{Advances in neural information processing systems}, 36, 2024{\natexlab{c}}.

\bibitem[Cutkosky \& Mehta(2020)Cutkosky and Mehta]{cutkosky2020momentum}
Cutkosky, A. and Mehta, H.
\newblock Momentum improves normalized sgd.
\newblock In \emph{International conference on machine learning}, pp.\  2260--2268. PMLR, 2020.

\bibitem[De~La~Harpe(1993)]{de1993hilbert}
De~La~Harpe, P.
\newblock On hilbert’s metric for simplices.
\newblock \emph{Geometric group theory}, 1:\penalty0 97--119, 1993.

\bibitem[Dubey et~al.(2024)Dubey, Jauhri, Pandey, Kadian, Al{-}Dahle, Letman, Mathur, Schelten, Yang, Fan, Goyal, Hartshorn, Yang, Mitra, Sravankumar, Korenev, Hinsvark, Rao, Zhang, Rodriguez, Gregerson, Spataru, Rozi{\`{e}}re, Biron, Tang, Chern, Caucheteux, Nayak, Bi, Marra, McConnell, Keller, Touret, Wu, Wong, Ferrer, Nikolaidis, Allonsius, Song, Pintz, Livshits, Esiobu, Choudhary, Mahajan, Garcia{-}Olano, Perino, Hupkes, Lakomkin, AlBadawy, Lobanova, Dinan, Smith, Radenovic, Zhang, Synnaeve, Lee, Anderson, Nail, Mialon, Pang, Cucurell, Nguyen, Korevaar, Xu, Touvron, Zarov, Ibarra, Kloumann, Misra, Evtimov, Copet, Lee, Geffert, Vranes, Park, Mahadeokar, Shah, van~der Linde, Billock, Hong, Lee, Fu, Chi, Huang, Liu, Wang, Yu, Bitton, Spisak, Park, Rocca, Johnstun, Saxe, Jia, Alwala, Upasani, Plawiak, Li, Heafield, and Stone]{llama3}
Dubey, A., Jauhri, A., Pandey, A., Kadian, A., Al{-}Dahle, A., Letman, A., Mathur, A., Schelten, A., Yang, A., Fan, A., Goyal, A., Hartshorn, A., Yang, A., Mitra, A., Sravankumar, A., Korenev, A., Hinsvark, A., Rao, A., Zhang, A., Rodriguez, A., Gregerson, A., Spataru, A., Rozi{\`{e}}re, B., Biron, B., Tang, B., Chern, B., Caucheteux, C., Nayak, C., Bi, C., Marra, C., McConnell, C., Keller, C., Touret, C., Wu, C., Wong, C., Ferrer, C.~C., Nikolaidis, C., Allonsius, D., Song, D., Pintz, D., Livshits, D., Esiobu, D., Choudhary, D., Mahajan, D., Garcia{-}Olano, D., Perino, D., Hupkes, D., Lakomkin, E., AlBadawy, E., Lobanova, E., Dinan, E., Smith, E.~M., Radenovic, F., Zhang, F., Synnaeve, G., Lee, G., Anderson, G.~L., Nail, G., Mialon, G., Pang, G., Cucurell, G., Nguyen, H., Korevaar, H., Xu, H., Touvron, H., Zarov, I., Ibarra, I.~A., Kloumann, I.~M., Misra, I., Evtimov, I., Copet, J., Lee, J., Geffert, J., Vranes, J., Park, J., Mahadeokar, J., Shah, J., van~der Linde, J., Billock, J., Hong, J., Lee, J., Fu,
  J., Chi, J., Huang, J., Liu, J., Wang, J., Yu, J., Bitton, J., Spisak, J., Park, J., Rocca, J., Johnstun, J., Saxe, J., Jia, J., Alwala, K.~V., Upasani, K., Plawiak, K., Li, K., Heafield, K., and Stone, K.
\newblock The llama 3 herd of models.
\newblock \emph{CoRR}, abs/2407.21783, 2024.

\bibitem[Dykstra(1983)]{dykstra1983algorithm}
Dykstra, R.~L.
\newblock An algorithm for restricted least squares regression.
\newblock \emph{Journal of the American Statistical Association}, 78\penalty0 (384):\penalty0 837--842, 1983.

\bibitem[Franklin \& Lorenz(1989)Franklin and Lorenz]{franklin1989scaling}
Franklin, J. and Lorenz, J.
\newblock On the scaling of multidimensional matrices.
\newblock \emph{Linear Algebra and its applications}, 114:\penalty0 717--735, 1989.

\bibitem[Gao et~al.(2023)Gao, Huang, Liu, Wang, Wang, Wang, Xu, and Yu]{gao2023eigenvalue}
Gao, K., Huang, Z.-H., Liu, X., Wang, M., Wang, S., Wang, Z., Xu, D., and Yu, F.
\newblock Eigenvalue-corrected natural gradient based on a new approximation.
\newblock \emph{Asia-Pacific Journal of Operational Research}, 40\penalty0 (01):\penalty0 2340005, 2023.

\bibitem[Gupta et~al.(2018)Gupta, Koren, and Singer]{gupta2018shampoopreconditionedstochastictensor}
Gupta, V., Koren, T., and Singer, Y.
\newblock Shampoo: Preconditioned stochastic tensor optimization, 2018.
\newblock URL \url{https://arxiv.org/abs/1802.09568}.

\bibitem[Hao et~al.(2024)Hao, Cao, and Mou]{Hao2024FloraLA}
Hao, Y., Cao, Y., and Mou, L.
\newblock Flora: Low-rank adapters are secretly gradient compressors.
\newblock \emph{ArXiv}, abs/2402.03293, 2024.
\newblock URL \url{https://api.semanticscholar.org/CorpusID:267412117}.

\bibitem[Hazan et~al.(2015)Hazan, Levy, and Shalev-Shwartz]{hazan2015beyond}
Hazan, E., Levy, K., and Shalev-Shwartz, S.
\newblock Beyond convexity: Stochastic quasi-convex optimization.
\newblock \emph{Advances in neural information processing systems}, 28, 2015.

\bibitem[Hu et~al.(2021)Hu, Shen, Wallis, Allen-Zhu, Li, Wang, Wang, and Chen]{hu2021lora}
Hu, E.~J., Shen, Y., Wallis, P., Allen-Zhu, Z., Li, Y., Wang, S., Wang, L., and Chen, W.
\newblock Lora: Low-rank adaptation of large language models.
\newblock \emph{arXiv preprint arXiv:2106.09685}, 2021.

\bibitem[Huang et~al.(2019)Huang, Zhou, Zhu, Liu, and Shao]{huang2019iterative}
Huang, L., Zhou, Y., Zhu, F., Liu, L., and Shao, L.
\newblock Iterative normalization: Beyond standardization towards efficient whitening.
\newblock In \emph{Proceedings of the IEEE/CVF conference on computer vision and pattern recognition}, pp.\  4874--4883, 2019.

\bibitem[Jordan et~al.(2024)Jordan, Jin, Boza, You, Cecista, Newhouse, and Bernstein]{jordan2024muon}
Jordan, K., Jin, Y., Boza, V., You, J., Cecista, F., Newhouse, L., and Bernstein, J.
\newblock Muon: An optimizer for hidden layers in neural networks, 2024.
\newblock URL \url{https://kellerjordan.github.io/posts/muon/}.

\bibitem[Kingma \& Ba(2015)Kingma and Ba]{adam}
Kingma, D.~P. and Ba, J.
\newblock Adam: {A} method for stochastic optimization.
\newblock In \emph{{ICLR} (Poster)}, 2015.

\bibitem[Korthikanti et~al.(2023)Korthikanti, Casper, Lym, McAfee, Andersch, Shoeybi, and Catanzaro]{korthikanti2023reducing}
Korthikanti, V.~A., Casper, J., Lym, S., McAfee, L., Andersch, M., Shoeybi, M., and Catanzaro, B.
\newblock Reducing activation recomputation in large transformer models.
\newblock \emph{Proceedings of Machine Learning and Systems}, 5:\penalty0 341--353, 2023.

\bibitem[Kunstner et~al.(2023)Kunstner, Chen, Lavington, and Schmidt]{kunstner2023noise}
Kunstner, F., Chen, J., Lavington, J.~W., and Schmidt, M.
\newblock Noise is not the main factor behind the gap between sgd and adam on transformers, but sign descent might be.
\newblock \emph{arXiv preprint arXiv:2304.13960}, 2023.

\bibitem[Kunstner et~al.(2024)Kunstner, Yadav, Milligan, Schmidt, and Bietti]{kunstner2024heavy}
Kunstner, F., Yadav, R., Milligan, A., Schmidt, M., and Bietti, A.
\newblock Heavy-tailed class imbalance and why adam outperforms gradient descent on language models.
\newblock \emph{arXiv preprint arXiv:2402.19449}, 2024.

\bibitem[Lewis \& Malick(2008)Lewis and Malick]{lewis2008alternating}
Lewis, A.~S. and Malick, J.
\newblock Alternating projections on manifolds.
\newblock \emph{Mathematics of Operations Research}, 33\penalty0 (1):\penalty0 216--234, 2008.

\bibitem[Li et~al.(2018)Li, Xie, Wang, and Gao]{li2018towards}
Li, P., Xie, J., Wang, Q., and Gao, Z.
\newblock Towards faster training of global covariance pooling networks by iterative matrix square root normalization.
\newblock In \emph{Proceedings of the IEEE conference on computer vision and pattern recognition}, pp.\  947--955, 2018.

\bibitem[Lialin et~al.(2023)Lialin, Shivagunde, Muckatira, and Rumshisky]{Lialin2023ReLoRAHT}
Lialin, V., Shivagunde, N., Muckatira, S., and Rumshisky, A.
\newblock Relora: High-rank training through low-rank updates.
\newblock In \emph{International Conference on Learning Representations}, 2023.
\newblock URL \url{https://api.semanticscholar.org/CorpusID:259836974}.

\bibitem[Ma et~al.(2024)Ma, Gong, Scetbon, and Meeds]{ma2024swansgdnormalizationwhitening}
Ma, C., Gong, W., Scetbon, M., and Meeds, E.
\newblock Swan: Sgd with normalization and whitening enables stateless llm training, 2024.
\newblock URL \url{https://arxiv.org/abs/2412.13148}.

\bibitem[Pierra(1984)]{pierra1984decomposition}
Pierra, G.
\newblock Decomposition through formalization in a product space.
\newblock \emph{Mathematical Programming}, 28:\penalty0 96--115, 1984.

\bibitem[Raffel et~al.(2020)Raffel, Shazeer, Roberts, Lee, Narang, Matena, Zhou, Li, and Liu]{2020t5}
Raffel, C., Shazeer, N., Roberts, A., Lee, K., Narang, S., Matena, M., Zhou, Y., Li, W., and Liu, P.~J.
\newblock Exploring the limits of transfer learning with a unified text-to-text transformer.
\newblock \emph{Journal of Machine Learning Research}, 21\penalty0 (140):\penalty0 1--67, 2020.
\newblock URL \url{http://jmlr.org/papers/v21/20-074.html}.

\bibitem[Rajbhandari et~al.(2020)Rajbhandari, Rasley, Ruwase, and He]{rajbhandari2020zero}
Rajbhandari, S., Rasley, J., Ruwase, O., and He, Y.
\newblock Zero: Memory optimizations toward training trillion parameter models.
\newblock In \emph{SC20: International Conference for High Performance Computing, Networking, Storage and Analysis}, pp.\  1--16. IEEE, 2020.

\bibitem[Rockafellar(1974)]{rockafellar1974conjugate}
Rockafellar, R.~T.
\newblock \emph{Conjugate duality and optimization}.
\newblock SIAM, 1974.

\bibitem[Sinkhorn(1964)]{sinkhorn1964relationship}
Sinkhorn, R.
\newblock A relationship between arbitrary positive matrices and doubly stochastic matrices.
\newblock \emph{The annals of mathematical statistics}, 35\penalty0 (2):\penalty0 876--879, 1964.

\bibitem[Sinkhorn \& Knopp(1967)Sinkhorn and Knopp]{sinkhorn1967concerning}
Sinkhorn, R. and Knopp, P.
\newblock Concerning nonnegative matrices and doubly stochastic matrices.
\newblock \emph{Pacific Journal of Mathematics}, 21\penalty0 (2):\penalty0 343--348, 1967.

\bibitem[Song et~al.(2022)Song, Sebe, and Wang]{song2022fast}
Song, Y., Sebe, N., and Wang, W.
\newblock Fast differentiable matrix square root and inverse square root.
\newblock \emph{IEEE Transactions on Pattern Analysis and Machine Intelligence}, 45\penalty0 (6):\penalty0 7367--7380, 2022.

\bibitem[Tibshirani(2017)]{tibshirani2017dykstra}
Tibshirani, R.~J.
\newblock Dykstra's algorithm, admm, and coordinate descent: Connections, insights, and extensions.
\newblock \emph{Advances in Neural Information Processing Systems}, 30, 2017.

\bibitem[Touvron et~al.(2023)Touvron, Lavril, Izacard, Martinet, Lachaux, Lacroix, Rozi{\`e}re, Goyal, Hambro, Azhar, et~al.]{touvron2023llama}
Touvron, H., Lavril, T., Izacard, G., Martinet, X., Lachaux, M.-A., Lacroix, T., Rozi{\`e}re, B., Goyal, N., Hambro, E., Azhar, F., et~al.
\newblock Llama: Open and efficient foundation language models.
\newblock \emph{arXiv preprint arXiv:2302.13971}, 2023.

\bibitem[Von~Neumann(1950)]{von1950functional}
Von~Neumann, J.
\newblock \emph{Functional operators: Measures and integrals}, volume~1.
\newblock Princeton University Press, 1950.

\bibitem[Watson(1992)]{watson1992characterization}
Watson, G.~A.
\newblock Characterization of the subdifferential of some matrix norms.
\newblock \emph{Linear Algebra Appl}, 170\penalty0 (1):\penalty0 33--45, 1992.

\bibitem[Xu et~al.(2024{\natexlab{a}})Xu, Xiang, Cai, and Wen]{xu2024adamlearningratescaling}
Xu, M., Xiang, L., Cai, X., and Wen, H.
\newblock No more adam: Learning rate scaling at initialization is all you need, 2024{\natexlab{a}}.
\newblock URL \url{https://arxiv.org/abs/2412.11768}.

\bibitem[Xu et~al.(2024{\natexlab{b}})Xu, Xiang, Cai, and Wen]{xu2024no}
Xu, M., Xiang, L., Cai, X., and Wen, H.
\newblock No more adam: Learning rate scaling at initialization is all you need.
\newblock \emph{arXiv preprint arXiv:2412.11768}, 2024{\natexlab{b}}.

\bibitem[You et~al.(2017)You, Gitman, and Ginsburg]{you2017lars}
You, Y., Gitman, I., and Ginsburg, B.
\newblock Large batch training of convolutional networks.
\newblock \emph{arXiv preprint arXiv:1708.03888}, 2017.

\bibitem[You et~al.(2019)You, Li, Reddi, Hseu, Kumar, Bhojanapalli, Song, Demmel, Keutzer, and Hsieh]{you2019lamb}
You, Y., Li, J., Reddi, S., Hseu, J., Kumar, S., Bhojanapalli, S., Song, X., Demmel, J., Keutzer, K., and Hsieh, C.-J.
\newblock Large batch optimization for deep learning: Training bert in 76 minutes.
\newblock \emph{arXiv preprint arXiv:1904.00962}, 2019.

\bibitem[Zhang \& Sennrich(2019)Zhang and Sennrich]{zhang2019root}
Zhang, B. and Sennrich, R.
\newblock Root mean square layer normalization.
\newblock \emph{Advances in Neural Information Processing Systems}, 32, 2019.

\bibitem[Zhang et~al.(2020)Zhang, Karimireddy, Veit, Kim, Reddi, Kumar, and Sra]{zhang2020adaptive}
Zhang, J., Karimireddy, S.~P., Veit, A., Kim, S., Reddi, S., Kumar, S., and Sra, S.
\newblock Why are adaptive methods good for attention models?
\newblock \emph{Advances in Neural Information Processing Systems}, 33:\penalty0 15383--15393, 2020.

\bibitem[Zhang et~al.(2024)Zhang, Chen, Li, Ding, Wu, Ye, Luo, and Sun]{zhang2024adam}
Zhang, Y., Chen, C., Li, Z., Ding, T., Wu, C., Ye, Y., Luo, Z.-Q., and Sun, R.
\newblock Adam-mini: Use fewer learning rates to gain more.
\newblock \emph{arXiv preprint arXiv:2406.16793}, 2024.

\bibitem[Zhao et~al.(2024{\natexlab{a}})Zhao, Zhang, Chen, Wang, Anandkumar, and Tian]{Zhao2024GaLoreML}
Zhao, J., Zhang, Z.~A., Chen, B., Wang, Z., Anandkumar, A., and Tian, Y.
\newblock Galore: Memory-efficient llm training by gradient low-rank projection.
\newblock \emph{ArXiv}, abs/2403.03507, 2024{\natexlab{a}}.
\newblock URL \url{https://api.semanticscholar.org/CorpusID:268253596}.

\bibitem[Zhao et~al.(2024{\natexlab{b}})Zhao, Morwani, Brandfonbrener, Vyas, and Kakade]{zhao2024deconstructing}
Zhao, R., Morwani, D., Brandfonbrener, D., Vyas, N., and Kakade, S.
\newblock Deconstructing what makes a good optimizer for language models.
\newblock \emph{arXiv preprint arXiv:2407.07972}, 2024{\natexlab{b}}.

\bibitem[Zhu et~al.(2024)Zhu, Zhang, Cong, Liu, Park, Chandra, Long, Pan, Wang, and Lee]{zhu2024apollo}
Zhu, H., Zhang, Z., Cong, W., Liu, X., Park, S., Chandra, V., Long, B., Pan, D.~Z., Wang, Z., and Lee, J.
\newblock Apollo: Sgd-like memory, adamw-level performance.
\newblock \emph{arXiv preprint arXiv:2412.05270}, 2024.

\end{thebibliography}
\bibliographystyle{theconference2024}

\newpage
\appendix
\onecolumn
\section*{Appendix}

\section{Implementation details}\label{app: code}

\paragraph{General setup} We describe the implementation setups for the LLM pre-training tasks. To enable a more straightforward and comparable analysis, we simply replicate the setting of \cite{Zhao2024GaLoreML}, under exactly the same model configs and optimizer hyperparameter configs, whenever possible. This includes the same model architecture, tokenizer, batch size, context length, learning rate scheduler, learning rates, subspace scaling, etc.  

\paragraph{Precision} All baselines uses BF16 for model weights, gradients, and optimizer states storage. For SWAN and SWAN$^\dag$, we follow the original paper and use FP32 in there whitening step.

\paragraph{Learning rate scheduling} we use exactly the same scheduler as in \cite{Zhao2024GaLoreML} for all methods.

\paragraph{Hyperparameters} Since \textbf{SinkGD} utilizes matrix-level operations on gradients, it can only be applied to 2D parameters. Therefore, in our experiments, we only apply \textbf{SinkGD} on all linear projection weights in transformer blocks. Similar to Galore \citep{Zhao2024GaLoreML}, the rest of the non-linear parameters still uses Adam as the default choice. Therefore, we follow the learning rate setup of Galore, where we fix some global learning rate across all model sizes and all modules. Then, for the linear projection modules where \textbf{SinkGD} is applied, we simply apply a scaling factor $\alpha$ on top of the global learning rate.  For all \textbf{SinkGD}, we adopt a \emph{lazy-tuning approach} (hyperparameters are set without extensive search), as detailed below. This helps to reduce the possibility of unfair performance distortion due to excessive tuning.   

\begin{itemize}
    \item \textbf{Adam}  For Adam we use same learning rate tuning procedure as suggested by \cite{Zhao2024GaLoreML} and \cite{ma2024swansgdnormalizationwhitening} (i.e., performing grid search over $\{0.01, 0.005, 0.001, 0.0005, 0.0001\}$). We found that the optimal learning rates for Adam is 0.001. The only exception is that for a model of size 1.3B: as we already know that a larger model requires smaller learning rates, we conduct a learning search for Adam over a smaller but more fine-grained grid of $\{ 0.001, 0.0007, 0.0005, 0.0003, 0.0001\}$. As a result, the optimal learning rate found for Adam on 1.3B is 0.0007.
    
    \item \textbf{SWAN}$^\dag$, is the tuned version of SWAN presented in \cite{ma2024swansgdnormalizationwhitening}. The original results of \textbf{SWAN} from \citet{ma2024swansgdnormalizationwhitening} assumes no learning rate warm-up and no learning rate tuning, in order to demonstrate the robustness of the method. This setting is more challenging than the setting of the usual literature \citep{Zhao2024GaLoreML, zhu2024apollo}. Hence, for fair comparison we relax those constraints and matches the setting of Galore and Apollo: we now allow learning rate warm-up (set to the same as Adam and Apollo), as well as larger learning rates for SWAN. This improved version of SWAN is denoted by \textbf{SWAN}$^\dag$. We use a global learning rate of 0.02, as well as the scaling factor $\alpha = 0.05$. This is selected by simply searching the learning rate over a constraint grid $\{0.01, 0.02, 0.05\}$, and then setting $\alpha = 0.05$ such that the effective learning rate is scaled back to 0.001. Finally, for other hyperparameters, we follow \cite{ma2024swansgdnormalizationwhitening}. 

    \item Finally, \textbf{SinkGD}, we use the same global learning rate of 0.02, as well as the scaling factor $\alpha = 0.05$ which are the same as \textbf{SWAN}$^\dag$, across all model sizes. We suspect with more careful tuning, its performance can be significantly improved; however, this is out of the scope of the paper. For $\texttt{SR-Sinkhorn}(\nabla,L)$ operation used in \textbf{SinkGD}, we simply use 5 steps.
\end{itemize}

\section{Proofs}

\subsection{Proof of Lemma~\ref{lem:properties-proj}}

\begin{proof}
Let us assume that $\Vert\mathcal{P}_{g}(x)\Vert_2=c$ and so for any $x$. Then we have that:
\begin{align*}
    \Vert \Vert\mathcal{P}_{g}\circ \mathcal{P}_{g}(x)\Vert_2 \Vert \mathcal{P}_{g}(x) \Vert_2  &\geq g^*(\mathcal{P}_g(x)):=\langle \mathcal{P}_{g}\circ \mathcal{P}_{g}(x),\mathcal{P}_{g}(x)\rangle \\
    &\geq \sup_{z:~g(z)\leq 1} \langle z, \mathcal{P}_g(x)\rangle 
\end{align*}
where the first inequality follows from Cauchy–Schwarz and the second inequality follows from the definition of $\mathcal{P}_g$. Now recall by definition, that $g(\mathcal{P}_g(x))\leq 1 $, and therefore we can select $z=\mathcal{P}_g(x)$ in the right inequality which gives:
\begin{align*}
     \Vert\mathcal{P}_{g}\circ \mathcal{P}_{g}(x)\Vert_2 \Vert \mathcal{P}_{g}(x) \Vert_2  &\geq g^*(\mathcal{P}_g(x))\\
    &\geq \Vert \mathcal{P}_g(x)\Vert_2^2 
\end{align*}
However because $\Vert\mathcal{P}_{g}\circ \mathcal{P}_{g}(x)\Vert_2 = \Vert\mathcal{P}_{g}(x)\Vert_2 = c$, we obtain that 
\begin{align*}
    g^*(\mathcal{P}_g(x))=\Vert \mathcal{P}_g(x)\Vert_2^2
\end{align*}
and by optimality, we also deduce that $\mathcal{P}_g\circ \mathcal{P}_g(x)=\mathcal{P}_g(x)$.
\end{proof}

\subsection{Proof of Thoerem~\ref{thm:cvg}}

\begin{proof}
First observe that thanks to Lemma~\ref{lem:properties-proj}, we have for any $n\geq 1$:
\begin{align}
\label{eq-proof-lem}
    \Vert x_n\Vert_2^2 = g_1^{*}(x_{2n-1})=g_2^{*}(x_{2n}) = c^2
\end{align}
where $g_1^*$ and $g_2^*$ are the dual norms of $g_1$ and $g_2$ respectively. We also have that for $n\geq 1$
\begin{align}
\label{eq-proof-ineq}
g_2(x_{2n}) \leq 1,~g_1(x_{2n+1})\leq 1
\end{align}
by definition of the normalized projections. We even have $$g_2(x_{2n})=g_1(x_{2n+1})=1$$ by optimality of the normalized projections. Let assume now that $n\geq 2$ is even, then we have that:
\begin{align*}
    \langle x_{n+1}, x_n\rangle &= \langle \mathcal{P}_{g_1}(x_{n}), x_n\rangle = g_1^{*}(x_n)\\
    &\geq \langle z ,x_n\rangle~\forall~z\in\mathcal{B}_1(0_d)
\end{align*}
where $\mathcal{B}_{g_1}(0_d)$ is the unit ball centered in $0_d$ associated to the norm $g_1$ and the inequality follows from the definition of $\mathcal{P}_{g_1}$. In particular by taking $z=x_{n-1}=\mathcal{P}_{g_1}(x_{n-2})\in\mathcal{B}_{g_1}(0_d)$, we obtain that:
\begin{align*}
    \langle x_{n+1}, x_n\rangle \geq \langle x_{n-1} ,x_n\rangle
\end{align*}
A similar proof can be conducted when $n$ is odd using the definition of $\mathcal{P}_{g_2}$. Therefore the sequence $(\langle x_{n+1}, x_n\rangle)_{n\geq 1}$ is increasing and bounded so it converges to a certain constant $r>0$. From this result we directly deduces that:
\begin{itemize}
    \item $(g_1^{*}(x_{2n}))_{n\geq 1}$ is monotonic increasing and converges towards $r$.
    \item $(g_2^{*}(x_{2n+1}))_{n\geq 1}$ is monotonic increasing and converges towards $r$.
\end{itemize}

Because $(x_{2n+1})_{n\geq 0}$ and $(x_{2n})_{n\geq 0}$ are bounded, we can extract a common subsequence $(x_{2\phi(n)+1})_{n\geq 1}$ and $(x_{2\phi(n)})_{n\geq 1}$ that converge to  some cluster points $x_1$ and $x_2$ respectively. 

Now by continuity of the dual norms and of the inner product we obtain that:
\begin{align*}
    \lim_{n\to\infty} g_2^{*}(x_{2\phi(n)+1})=g_2^{*}(x_1)\\
    \lim_{n\to\infty} g_1^{*}(x_{2\phi(n)})=g_1^{*}(x_2)\\
    \lim_{n\to\infty}\langle x_{2\phi(n)+1}, x_{\phi(n)}\rangle = \langle x_1, x_2\rangle 
\end{align*}
However observe that these three sequences are subsequences of $(\langle x_n,x_{n+1}\rangle)_{n\geq 0}$ which converges towards $r$, therefore we obtain that:
\begin{align*}
    r = g_2^{*}(x_1) = g_1^{*}(x_2) = \langle x_1, x_2\rangle 
\end{align*}
Additionally, remark that 
\begin{align}
\label{eq-proof=1}
g_2^{*}(x_{2\phi(n)+1})=g_2^*(\mathcal{P}_{g_1}(x_{2\phi(n)}))
\end{align}

Let us now show that $x_{2\phi(n)+1}=\mathcal{P}_{g_1}(x_{2\phi(n)})\xrightarrow[n\to\infty]{}\mathcal{P}_{g_1}(x_2)$. Indeed we have that:
\begin{align*}
    \langle \mathcal{P}_{g_1}(x_{2\phi(n)}), x_{2\phi(n)}\rangle = g_1^{*}(x_{2\phi(n)})\xrightarrow[n\to\infty]{}g_1^{*}(x_2)
\end{align*}
Then, because $(\mathcal{P}_{g_1}(x_{2\phi(n)}))_{n\geq 0}$ is bounded, we can extract a subsequence that converges towards $z$ such that $g_1(z)\leq 1$, from which follows that:
\begin{align*}
    \langle z,  x_2\rangle = \langle \mathcal{P}_{g_1}(x_2),x_2\rangle
\end{align*}
then by optimality of $\mathcal{P}_{g_1}(x_2)$ over the unit ball induced by $g_1$, we deduce that $z= \mathcal{P}_{g_1}(x_2)$. This is true for all converging sub-sequences of $(\mathcal{P}_{g_1}(x_{2\phi(n)}))_{n\geq 0}$, therefore we have that $\mathcal{P}_{g_1}(x_{2\phi(n)})\xrightarrow[n\to\infty]{}\mathcal{P}_{g_1}(x_2)$, and by unicity of the limit, it follows that 
$$x_1 =\mathcal{P}_{g_1}(x_2)\; .$$
Now from the equality $g_2^{*}(x_1) =  \langle x_1, x_2\rangle $, and given the fact that $g_2(x_2)\leq 1$ (as for all $n$ $g_2(x_{2\phi(n)})\leq 1$ which is obtained from~\eqref{eq-proof-ineq}), we deduce that 
$$x_2 = \mathcal{P}_{g_2}(x_1)$$
thanks to the optimality of $\mathcal{P}_{g_2}$. Now observe now that:
\begin{align*}
    g_2^*(x_1) = g_2^*(\mathcal{P}_{g_1}(x_2)) &= \langle \mathcal{P}_{g_2}\circ  \mathcal{P}_{g_1}(x_2), \mathcal{P}_{g_1}(x_2)\rangle\\
    & = \langle x_2, \mathcal{P}_{g_2}(x_2)\rangle 
\end{align*}
where the equality follows from the fact that:
\begin{align*}
    \mathcal{P}_{g_2}\circ  \mathcal{P}_{g_1}(x_2) = \mathcal{P}_{g_2}(x_1) = x_2
\end{align*}
and the two equalities follows the previous results obtained.
Therefore we obtain that
\begin{align*}
    g_2^*(x_1)  = g_2^{*}(x_2)=c^2
\end{align*}
where the last equality follows from~\eqref{eq-proof-lem}. Thus, we obtain that
\begin{align*}
 r = c^2 = g_2^*(x_1) =\langle x_1,x_2\rangle \leq \Vert x_1\Vert_2 \Vert x_2\Vert_2 
\end{align*}
but from~\eqref{eq-proof-lem}, $\Vert x_1\Vert_2=\Vert x_2\Vert_2=c$, from which follows that $x_1=x_2=x$, and $\mathcal{P}_{g_1}(x)=\mathcal{P}_{g_2}(x)=x$. As a by-product, we also obtain that $\langle x_n, x_{n+1}\rangle \xrightarrow[n\to\infty]{} r=c^2$, and therefore $\Vert x_n - x_{n-1}\Vert_2^2 = 2c^2 - 2\langle x_n,x_{n+1}\rangle \xrightarrow[n\to\infty]{}0$.

From the above proof, we also conclude that if $y$ is a cluster point of $(x_n)_{n\geq 0}$, then there exists $\psi$ such that $(x_{\psi(n)})_{n\geq 0}$ converges towards $y$ that satisfies:
$\mathcal{P}_{g_1}(y)=\mathcal{P}_{g_2}(y)=y$. Indeed this follows simply from the fact that we can extract a subsequence of $(x_{\psi(n)})_{n\geq 0}$ which has all indices that are either even or odd.

Let us now show that 
\begin{align*}
    g_1(x_n)\xrightarrow[n\to\infty]{} 1,~\text{and}~
    g_2(x_n)\xrightarrow[n\to\infty]{} 1\; .
\end{align*}

Indeed for a convergent subsequence, if the subsequence has infinitely many odd indices the result is trivial from the fact $g_1(x_{2n+1})=1$. Now if the indices are even, we obtain that $g_1(x_{2\phi(n)})\xrightarrow[n\to\infty]{}g_1(x)$, however $x$ has to be a fixed-point so $g_1(x)=g_1(\mathcal{P}_{g_1}(x))=1$. This hold for any subsequences, therefore we have $g_1(x_n)\xrightarrow[n\to\infty]{}1$. Similarly, we can apply the same reasoning for $g_2(x_n)$.

Let us now show the following Lemma.

\begin{lemma}
\label{lem-set-view}
Let $g_1$ and $g_2$ two norms satisfying the same assumption as in Theorem~\ref{thm:cvg}, that is for all $x$, $\Vert \mathcal{P}_{g_1}(x) \Vert_2 = \Vert \mathcal{P}_{g_2}(x)\Vert_2 = c$ with $c>0$. Then by denoting $\mathcal{S}_{g}$ the unit sphere associated to a norm $g$, we have:
\begin{align*}
    \mathcal{S}_{g_1}\cap \mathcal{S}_{g_2}\cap\mathcal{S}_{c\ell_2} = \mathcal{F}~~\text{where  }\mathcal{F}:=\{x:~\mathcal{P}_{g_1}(x)=\mathcal{P}_{g_2}(x)=x\}\; .
\end{align*}
\end{lemma}
\begin{proof}
Indeed $\mathcal{F}\subset \mathcal{S}_{g_1}\cap \mathcal{S}_{g_2}\cap\mathcal{S}_{c\ell_2} $ follows directly from the definition of $\mathcal{P}_{g_1}$, $\mathcal{P}_{g_1}$, and from Assumption~\ref{assump-norm}. Now let $z\in  \mathcal{S}_{g_1}\cap \mathcal{S}_{g_2}\cap\mathcal{S}_{c\ell_2}$. Observe that
\begin{align*}
    c^2\geq \langle z, \mathcal{P}_{1}(z)\rangle = \sup_{q:~g_1(q)=1}\langle z, q\rangle 
\end{align*}
where the inequality follows from the assumption on $\mathcal{P}_{g_1}$ and from the definition of $z$. Then as $g_1(z)=1$, we deduce that:
\begin{align*}
     c^2\geq \langle z, \mathcal{P}_{1}(z)\rangle \geq \Vert z\Vert_2^2=c^2
\end{align*}
from which folows that $\mathcal{P}_{g_1}(z)=z$. Similarly we deduce that $\mathcal{P}_{g_2}(z)=z$, and thus we have $\mathcal{S}_{g_1}\cap \mathcal{S}_{g_2}\cap\mathcal{S}_{c\ell_2}\subset \mathcal{F} $.
\end{proof}

Now observe that $d(x_n,\mathcal{S}_{g_1})\xrightarrow[n\to \infty]{} 0 $, and $d(x_n,\mathcal{S}_{g_1})\xrightarrow[n\to \infty]{} 0 $. Additionally, from~\eqref{eq-proof=1}, we have $d(x_n,\mathcal{S}_{c\ell_2})=0$, therefore we have that $d(x,\mathcal{S}_{g_1}\cap \mathcal{S}_{g_2}\cap\mathcal{S}_{c\ell_2})\xrightarrow[n\to\infty]{}0$ since all these spaces are closed, and the result follows from  Lemma~\ref{lem-set-view}.  

\end{proof}

\section{On the Convex Relaxation of Problem~\eqref{eq:multi-norm-opt}}
\label{sec:convex-relaxation}
Given $K$ norms, $(g_1,\dots, g_K)$, in this section we are interested in solving:
\begin{align}
\label{eq:multi-norm-opt-convex-app}
 \argmax_{z} \langle \nabla, z\rangle\quad \text{s.t.}~\forall~~i\in [|1,K|],~~g_i(z)\leq 1
\end{align}
which as stated in the main paper is equivalent to solve
\begin{align*}
 \argmax_{z} \langle \nabla, z\rangle\quad \text{s.t.}~\Vert z\Vert\leq 1
\end{align*}
where
\begin{align}
\label{eq:def-gen-norm}
\Vert z\Vert := \max_{i\in[|1,K|]}g_i(z)\; .
\end{align}

Note that this constrained optimization problem is exactly finding the subdifferential of the dual norm $\Vert \cdot \Vert$. To see this, let us recall the following Lemma with its proof~\cite{watson1992characterization}.
\begin{lemma}
The subdifferential of a norm $\Vert \cdot \Vert$ at $x$ is given by
\begin{align*}
    \partial\Vert \cdot\Vert(x)=\{p\in\mathbb{R}^d:~\Vert p\Vert_*\leq 1\text{, } \langle p, x\rangle = \Vert x\Vert\}
\end{align*}
where the dual norm is defined as
\begin{align*}
    \Vert x\Vert_*:=\max_{\Vert z\Vert\leq 1}\langle z, x\rangle 
\end{align*}
\end{lemma}

\begin{proof}
We can show this result by double inclusion. Let us define the subdifferential of a norm as
\begin{align*}
    \partial\Vert \cdot\Vert(x):=\{p:~\Vert y\Vert \geq \Vert x \Vert + \langle p, y-x\rangle~\forall y\}
\end{align*}
and let us denote our set of interest as
\begin{align*}
    \mathcal{V}(x):=\{p\in\mathbb{R}^d:~\Vert p\Vert_*\leq 1\text{, } \langle p, x\rangle = \Vert x\Vert\}
\end{align*}

Let $p\in\mathcal{V}(x)$. Then we have
\begin{align*}
    \Vert x \Vert + \langle p, y-x\rangle &= \langle p, y\rangle \\
    &\leq \Vert p\Vert_* \Vert y\Vert\\
    &\leq \Vert y\Vert
\end{align*}
where the first equality comes from the definition of $p$, the first inequality comes from Holder, and the last one is obtained by definition of $p$. So we deduce that $\mathcal{V}(x)\subset  \partial\Vert \cdot\Vert(x)$.
Let us now take $p\in \partial\Vert \cdot\Vert(x)$, that is $p$ such that for all $y$ 
$\Vert y\Vert \geq \Vert x \Vert + \langle p, y-x\rangle$. Then we have for all $y$ that:
\begin{align*}
    \langle p,x\rangle - \Vert x\Vert &\geq \langle p,y\rangle 
 - \Vert y\Vert\\
 & \geq \sup_y\langle p,y\rangle 
 - \Vert y\Vert\\
 &\geq \Vert p\Vert^{*}
\end{align*}
where $\Vert \cdot\Vert^{*}$ is the Fenchel-Legendre transform of the norm $\Vert\cdot \Vert$. From Lemma~\ref{lemma-fenchel-norm}, we deduce that
\begin{align*}
    \langle p,x\rangle - \Vert x\Vert \geq \bm{1}_{\mathcal{B}_{1}}(p)
\end{align*}
where $\mathcal{B}_{1}$ is the unit ball associated with the dual norm $\Vert \cdot\Vert_{*}$. As the left-hand side is finite, we deduce that $p\in\mathcal{B}_1$. Then we deduce that 
\begin{align*}
   \Vert x\Vert \geq  \langle p, x\rangle \geq \Vert x\Vert
\end{align*}
where the left inequality is obtained by applying Holder to the inner-product, from which follows the result.

\end{proof}

For simple norms, such as the $\ell_p$-norm with $1< p\leq +\infty$, obtaining an element of $\partial\Vert \cdot\Vert_{*}(\nabla)$ can be done in closed-form. For $\ell_p$ norms, recall the their dual norm are the $\ell_q$ norms with $q$ the dual exponent respectively. The following Lemma provide analytic formulas for such norms. 

\begin{lemma}
For $x\in\mathbb{R}^d$, let us define the $\ell_p$-norm as 
$$\Vert x\Vert_p:=\left(\sum_{i=1}^d |x_i|^p\right)^{1/p}$$

\textbf{Case 1: \boldmath{$1 < p < \infty$}.}
Let us define the dual exponent $q$ by
\[
\frac{1}{p} + \frac{1}{q} = 1.
\]
Then the subdifferential of $\|x\|_p$ is
\[
\partial \|x\|_p \;=\;
\begin{cases}
\displaystyle
\left\{
  \frac{\bigl(|x_1|^{p-2}x_1,\dots,|x_d|^{p-2}x_d\bigr)}{\|x\|_p^{\,p-1}}
\right\}, 
& x \neq 0,
\\
\bigl\{\,g \in \mathbb{R}^d : \|g\|_q \,\le\, 1 \bigr\}, 
& x = 0.
\end{cases}
\]

\textbf{Case 2: \boldmath{$p = 1$}.}
For the $\ell_1$-norm, the subdifferential at $x \in \mathbb{R}^n$ is given by
\[
\partial \|x\|_1 
\;=\;
\Bigl\{
   g \in \mathbb{R}^n : 
   g_i = \mathrm{sign}(x_i)
\Bigr\}.
\]
Here, $\mathrm{sign}(x_i)$ is $+1$ if $x_i > 0$, $-1$ if $x_i < 0$, and can be any value in $[-1,1]$ if $x_i=0$.

\textbf{Case 3: \boldmath{$p = \infty$}.}  
For the $\ell^\infty$-norm, let $M = \|x\|_\infty$ and let us define
\begin{align*}
\mathcal{S}(x):=\{
  g: 
  g_i=0 \text{ if } |x_i|<M,\,
  g_i = \mathrm{sign}(x_i)~\text{else},\,
  \|g\|_1 = 1
\}
\end{align*}

Then by denoting for any set $A\subset\mathbb{R}^d$, $\text{conv}(A)$ the convex hull of the set $A$, we have:
\[
\partial \|x\|_\infty
\;=\;
\begin{cases}
\displaystyle
\mathrm{conv}\!(\mathcal{S}(x))
& x \neq 0,
\\
\bigl\{\, g \in \mathbb{R}^n : \|g\|_1 \le 1 \bigr\},
& x = 0.
\end{cases}
\]
\end{lemma}

However, for general norms, there are not known closed-form solutions of their associated subdifferentials. In particular, if the norm is defined as in~\eqref{eq:def-gen-norm}, even when the $g_i$'s are simple norms (i.e. norms for which we can compute the subdifferential of their dual norms), then no closed-form solution can be obtained in general.

\subsection{A Dual Perspective}

In this section, we propose an algorithmic approach to solve the convex relaxation of the problem introduced in~\eqref{eq:multi-norm-opt}. More formally, given a family of simple norms $(g_i)_{i=1}^K$ and some positive constants $(\varepsilon_i)_{i=1}^K$, we consider the following problem:
\begin{align}
\label{eq:proj-gen-opt}
 \max_{d\theta\in\mathbb{R}^d} \langle \nabla, d\theta\rangle\quad \text{s.t.}~g(d\theta)\leq 1\; .
\end{align}
where
\begin{align*}
g(x):= \max_{i\in[|1,K|]}\frac{g_i(x)}{\varepsilon_i}
\end{align*}
which is also a norm. For such problems, as long as $\nabla\neq 0$, then the solutions lies in the level set $\{d\theta:~g(d\theta)=1\}$. Even if the subdifferentials of (the dual norm of) each $g_i$ can be derived in closed form, there is not known closed-form for the subdifferential of (the dual norm of) $g$. To solve~\eqref{eq:proj-gen-opt}, we propose to consider a coordinate gradient descent on the dual. A simple application of the Fenchel duality~\cite{rockafellar1974conjugate} leads to the following equivalent optimization problem:
\begin{align}
\label{dual-gen-proj}
\inf_{\lambda_1,\dots,\lambda_K}\sum_{i=1}^K \epsilon_i g^\dagger_i(\lambda_i)\quad\text{s.t.}\quad \nabla\mathcal{L}(\theta) = \sum_{i=1}^K \lambda_i
\end{align}
where $g_i^{\dagger}$ is the dual norm of $g_i$ and so for all $i\in[|1,K|]$, from which a primal solution can be recovered by simply finding $y_i$ s.t. $\lambda_i>0$ and such that $\langle \lambda_i, y_i\rangle=\varepsilon_ig_i^{\dagger}(y_i)$ under the condition that $g_i(y_i)=\varepsilon_i$, which is equivalent to solve:
\begin{align*}
    y_i^{*}:=\varepsilon_i\argmax_{z:~g_i(z)\leq 1} \langle z,\lambda_i\rangle \; .
\end{align*}

\begin{proof}
Let $(\mathcal{B}_i(\epsilon_i))_{i=1}^K$ the ball associated with the norm $(g_i)_{i=1}^K$ with radius $(\varepsilon_i)_{i=1}^K$ respectively. Let us also denote for any set $\mathcal{A}\subset\mathbb{R}^d$, the indicator function as 
$$\bm{1}_{A}(x)=\begin{cases} 
          0 ~~\text{if}~~x\in A\\
          +\infty ~~\text{otherwise}
       \end{cases}$$

In the following we denote $f(x):=\langle x, \nabla\rangle$. Then~\eqref{eq:proj-gen-opt} can be reformulated as the following optimization problem:
\begin{align*}
    -\inf_{d\theta} f(d\theta) +\sum_{i=1}^K \bm{1}_{\mathcal{B}_i(\epsilon_i)}(
d\theta)
\end{align*}
which can be again reparameterized (up to the sign) as 
\begin{align*}
    \inf_{x=y_i, ~ \forall i\in[|1,K|]} f(x) +\sum_{i=1}^K \bm{1}_{\mathcal{B}_i(\epsilon_i)}(y_i)
\end{align*}
Now the Lagrangian associated with this problem is:
\begin{align*}
    &\mathcal{F}((\lambda_i)_{i=1}^{K}, (y_i)_{i=1}^K, x):= \\
    & f(x) - \langle x,\sum_{i=1}^K\lambda_i\rangle  +\sum_{i=1}^K \bm{1}_{\mathcal{B}_i(\epsilon_i)}(y_i) + \langle y_i, \lambda_i\rangle 
\end{align*}

And taking the infimum of the Lagrangian w.r.t the primal variables leads to the following optimization problem:
\begin{align*}
    \inf_{x} f(x) - \langle x,\sum_{i=1}^K\lambda_i\rangle + \sum_{i=1}^K \inf_{y_i} \bm{1}_{\mathcal{B}_i(\epsilon_i)}(y_i) + \langle y_i, \lambda_i\rangle 
\end{align*}

Now observe that
\begin{align*}
    \inf_{x} f(x) - \langle x,\sum_{i=1}^K\lambda_i\rangle &= -\sup_{x}\langle x,\sum_{i=1}^K\lambda_i\rangle - f(x)\\
    &=-f^*(\sum_{i=1}^K\lambda_i)
\end{align*}
where $f^*$ is the Fenchel-Legendre transform of $f$. Similarly, we have:
\begin{align*}
     \inf_{y_i} \bm{1}_{\mathcal{B}_i(\epsilon_i)}(y_i) + \langle y_i, \lambda_i\rangle &= -\sup_{y_i}\langle y_i,-\lambda_i\rangle - \bm{1}_{\mathcal{B}_i(\epsilon_i)}(y_i)\\
    &=-\bm{1}_{\mathcal{B}_i(\epsilon_i)}^*(-\lambda_i)
\end{align*}
Finally the dual of the problem is:
\begin{align*}
   \sup_{\lambda_1,\dots,\lambda_K} -f^*(\sum_{i=1}^K\lambda_i) -\sum_{i=1}^K \bm{1}_{\mathcal{B}_i(\epsilon_i)}^*(-\lambda_i)
\end{align*}
Now recall that $f(x):=\langle x, \nabla\rangle$, therefore we have that
\begin{align*}
    f^*(x)=\bm{1}_{\{\nabla\}}(x)
\end{align*}
Also, we have that
\begin{align*}
\bm{1}_{\mathcal{B}_i(\epsilon_i)}^*(x)=\varepsilon_i g_i^\dagger(x)
\end{align*}
where $g_i^\dagger$ is the dual norm of $g_i$, from which it follows the final dual formulation:
\begin{align*}
\inf_{\lambda_1,\dots,\lambda_K}\sum_{i=1}^K \epsilon_i g^\dagger_i(\lambda_i)\quad\text{s.t.}\quad \nabla\mathcal{L}(\theta) = \sum_{i=1}^K \lambda_i.
\end{align*}
Finally, Slater condition are verified, thus strong duality holds, and the KKT conditions gives the following primal-dual conditions:
\begin{align*}
    \begin{cases} 
          \nabla\mathcal{L}(\theta)=\sum_{i=1}^K \lambda_i\\
          \lambda_i\in\partial\bm{1}_{\mathcal{B}_i(\varepsilon_i)}(y_i)~~\forall i\\
          x=y_i~~\forall i
       \end{cases}
\end{align*}
Now according to Lemma~\ref{lemma-subdiff-indicator}, we have that

\begin{align*}
\partial\bm{1}_{\mathcal{B}_i(\varepsilon_i)}(x)=\begin{cases} 
          \{0\}\text{ if } g_i(x)< \varepsilon_i\\
          \emptyset \text{ if } g_i(x)> \varepsilon_i\\
          \{p:~\langle p, x\rangle =\varepsilon_ig_i^{\dagger}(p)\} \text{ if } g_i(x)=\varepsilon_i
       \end{cases}
\end{align*}
from which follows that one can recover a primal solution by simply finding $y_i$ s.t. $\lambda_i>0$ and such that $\langle \lambda_i, y_i\rangle=\varepsilon_ig_i^{\dagger}(y_i)$ under the condition that $g_i(y_i)=\varepsilon_i$, which is equivalent to solve:
\begin{align*}
    y_i^{*}:=\varepsilon_i\argmax_{z:~g_i(z)\leq 1} \langle z,\lambda_i\rangle \; .
\end{align*}
\end{proof}

\begin{algorithm}[tb]
   \caption{Primal-Dual Algorithm to solve~\eqref{update-cgd}}
   \label{alg:primal-dual}
\begin{algorithmic}
   \STATE {\bfseries Input:} $\beta_k\in\mathbb{R}^d$,   $\epsilon_1,\epsilon_k>0$, $\eta_1, \eta_2>0$, s.t. $\eta_1\eta_2<1$.
   \STATE Initialize $\lambda=z=u=0_d$.
   \FOR{$i=1$  {\bfseries to} $L$}
   \STATE $\lambda_{\text{old}} \gets \lambda$
    \STATE $z\gets \text{proj}_{\mathcal{B}_{1}(\epsilon_1)}(z + \eta_1 (u - \beta_k)) $
    \STATE $\lambda\gets \text{prox}_{\eta_2\epsilon_k g_k^{\dagger}}(\lambda - \eta_2 z)$
    \STATE $u\gets 2\lambda - \lambda_{\text{old}}$
   \ENDFOR
   \STATE Return $\lambda$
\end{algorithmic}
\end{algorithm}

To solve the dual problem introduced in~\eqref{dual-gen-proj}, we apply a coordinate gradient descent on the $\lambda_i$. More precisely, we can reformulate the problem as an unconstrained optimization one by considering:
\begin{align*}
\inf_{\lambda_2,\dots,\lambda_K} \epsilon_1 g_1^{\dagger}\left(\nabla\mathcal{L}(\theta) - \sum_{i=2}^K\lambda_i\right) +\sum_{i=2}^K \epsilon_i g^\dagger_i(\lambda_i)
\end{align*}
Starting with $\lambda_2^{(0)}=\dots=\lambda_K^{(0)}=0_d$, we propose to apply the following updates at time $t\geq 0$ and so for all $k\in[|2,K|]$:
\begin{align}
\label{update-cgd}
\lambda_{k}^{(t+1)}=\argmin_{\lambda_k} \epsilon_1 g_1^{\dagger}\left(\beta_k^{(t)}- \lambda_k \right) + \epsilon_k g^\dagger_k(\lambda_k)
\end{align}
where $\beta_k^{(t)}:=\nabla\mathcal{L}(\theta) - \sum\limits_{i\neq k}\lambda_i^{(t)}$. In order to solve~\eqref{update-cgd}, we leverage the so-called Chambolle-Pock algorithm~\cite{chambolle2011first}. Let us denote $h_1(\lambda):=\varepsilon_1 g_1^{\dagger}(\beta_k^{(t)}-\lambda)$ and $h_k(\lambda):=\varepsilon_k g_k^{\dagger}(\lambda)$. Then we can write
\begin{align*}
    \inf_{\lambda} h_1(\lambda) + h_k(\lambda) = \inf_{\lambda} h_k(\lambda) + \sup_{z} \langle z, \lambda\rangle - h_1^*(z)\\
    = \inf_{\lambda}\sup_{z} \langle z,\lambda\rangle - h_1^{*}(z) + h_k(\lambda)
\end{align*}
where $h_1^*$ is the Fenchel-Legendre transform of $h_1$ given by $h_1^{*}(x)=\langle x, \beta_k^{(t)}\rangle + \bm{1}_{\mathcal{B}_{1}(\epsilon_1)}(x)$
where $\mathcal{B}_{1}(\epsilon_1)$ is the ball induced by the norm $g_1$ of radius $\varepsilon_1$. We are now ready to present the Chambolle-Pock algorithm for our setting as presented in Algorithm~\ref{alg:primal-dual}. This algorithm requires to have access to the projection operation w.r.t the norm $g_1$ and the proximal operator w.r.t the norm $g_k^{\dagger}$, that is, it requires to have access to:
\begin{align*}
    \text{proj}_{\mathcal{B}_1(\varepsilon_1)}(x)&:=\argmin_{z:~g_1(z)\leq \varepsilon_1}\Vert z - x\Vert_2\\
    \text{prox}_{\lambda g_k^{\dagger}}(x)&:=\argmin_{z}\frac{\Vert z - x\Vert_2^2}{2} + \lambda g_k^{\dagger}(z)
    \end{align*}

Computing proximal and projection operators of norms and their duals can also be done using the Moreau decomposition property which states that:
\begin{align*}
    \text{prox}_{f}(x) + \text{prox}_{f^*}(x) = x
\end{align*}
in particular if $f:=\Vert \cdot \Vert$ is a norm, we have:
\begin{align*}
    \text{prox}_{\Vert \cdot \Vert}(x) + \text{proj}_{\mathcal{B}_{*}(1)}(x) = x
\end{align*}
where $\mathcal{B}_{*}(1)$ is the unit ball of the dual norm of $\Vert \cdot \Vert$. Finally, the full coordinate gradient scheme is presented in Algorithm~\ref{alg:cd-relax} which returns a solution of the primal problem defined in~\eqref{eq:multi-norm-opt-convex-app}.

\begin{algorithm}[tb]
   \caption{Coordinate Gradient Descent to solve~\eqref{dual-gen-proj}}
   \label{alg:cd-relax}
\begin{algorithmic}
   \STATE {\bfseries Input:} the gradient$\nabla\mathcal{L}(\theta)$ and  $\epsilon_1,\dots,\epsilon_K>0$
   \STATE Initialize $\lambda_2=\dots=\lambda_K=0_d$.
   \FOR{$t=1$  {\bfseries to} $T$}
   \FOR{$k=2$ {\bfseries to} $K$}
   \STATE $\beta_k^{(t)}\gets\nabla\mathcal{L}(\theta) - \sum\limits_{i\neq k}\lambda_i^{(t)}$
   \STATE $\lambda_k^{(t+1)} \gets \argmin_{\lambda} h_1(\lambda) + h_k(\lambda)$~~with Alg.~\ref{alg:primal-dual}
   \ENDFOR
    \ENDFOR
    \STATE Find $k$ such that $\lambda_k>0$
    \STATE Return $x^*:=\varepsilon_k\argmax\limits_{z:~g_k(z)\leq 1} \langle z,\lambda_k\rangle$
\end{algorithmic}
\end{algorithm}

\begin{lemma}
\label{lemma-fenchel-norm}
Let $\Vert\cdot\Vert$ be a norm on $\mathbb{R}^d$ with dual norm $\Vert x\Vert_{*}:=max_{z:\Vert z\Vert\leq 1} \langle z, x\rangle$, then the Fenchel-Legendre transform of $\Vert\cdot\Vert$ is the indicator function of the unit ball induced by its dual norm. More formally, we have
\begin{align*}
    \sup_{z\in\mathbb{R}^d}\langle z, x\rangle - \Vert z\Vert =  \begin{cases} 
          0 ~~\text{if}~~ \Vert x\Vert_* \leq 1\\
          +\infty ~~\text{otherwise}
       \end{cases}
\end{align*}
\end{lemma}
\begin{proof}
Using the fact that $\Vert x\Vert=\sup_{z:\Vert z\Vert_{*}\leq 1} \langle z,x\rangle$, we have:
\begin{align*}
   \sup_{z\in\mathbb{R}^d}\langle z, x\rangle - \Vert z\Vert & = \max_{z\in\mathbb{R}^d}\langle z, x\rangle  - \sup_{y:\Vert y\Vert_{*}\leq 1} \langle y,z\rangle\\
   &= \sup_{z\in\mathbb{R}^d}\inf_{y:\Vert y\Vert_{*}\leq 1} \langle z, x-y\rangle \\
   &=\inf_{y:\Vert y\Vert_{*}\leq 1}\sup_{z\in\mathbb{R}^d}\langle z, x-y\rangle\\
   &=\inf_{y:\Vert y\Vert_{*}\leq 1} \begin{cases} 
          0 ~~\text{if}~~y=x\\
          +\infty ~~\text{otherwise}
       \end{cases}
\end{align*}
which gives the desired result. Note that the third equality follows from Sion's minimax theorem. 
\end{proof}

\begin{lemma}
\label{lemma-subdiff-indicator}
Let $\Vert \cdot\Vert$ a norm on $\mathbb{R}^d$ and $\varepsilon>0$. Then we have:
\begin{align*}
\partial\bm{1}_{\mathcal{B}(\varepsilon)}(x)=\begin{cases} 
          \{0\}\text{ if } \Vert x\Vert< \varepsilon\\
          \emptyset \text{ if } \Vert x\Vert> \varepsilon\\
          \{p:~\langle p, x\rangle =\varepsilon\Vert p\Vert_*\} \text{ if } \Vert x\Vert=\varepsilon
       \end{cases}
\end{align*}
where $\Vert \cdot \Vert_*$ is the dual norm of $\Vert \cdot \Vert$, and $\mathcal{B}(\varepsilon)$ is the ball of radius $\varepsilon$ w.r.t the norm $\Vert \cdot \Vert$.
\end{lemma}

\begin{proof}
Recall that the definition of the subdifferential is:
\begin{align*}
    \partial\bm{1}_{\mathcal{B}(\varepsilon)}(x):=\{p:~\bm{1}_{\mathcal{B}(\varepsilon)}(y)\geq \bm{1}_{\mathcal{B}(\varepsilon)}(x) + \langle p, y-x\rangle~~\forall y\}
\end{align*}
If $\Vert x\Vert < \varepsilon$, then we have that $p$ must satisfy for all $y\in\mathcal{B}(\varepsilon)$:
\begin{align*}
    \langle p, y-x\rangle\leq 0
\end{align*}
By taking $\gamma$ sufficiently small we can therefore choose $y=x+\gamma \frac{p}{\Vert p \Vert_2}\in\mathcal{B}(\varepsilon)$ which leads to 
\begin{align*}
    \gamma \Vert p\Vert_2\leq 0  
\end{align*}
which is only true for $p=0$ as $\gamma$ can be selected to be negative or positive. Now if $\Vert x\Vert > \varepsilon$, then the subdifferential is clearly empty. Finally, let us consider the case where $\Vert x\Vert = \varepsilon$. We deduce that:
\begin{align*}
    \langle p, x\rangle \geq \langle p, y\rangle -\bm{1}_{\mathcal{B}(\varepsilon)}(y)
\end{align*}
and so for all $y$. Therefore we obtain that
\begin{align*}
    \langle p, x\rangle  &\geq \sup_y \langle p, y\rangle -\bm{1}_{\mathcal{B}(\varepsilon)}(y)\\
    &=\varepsilon\Vert p\Vert_*
\end{align*}
But we also have that:
\begin{align*}
    \varepsilon\Vert p\Vert_*=\Vert p\Vert_* \Vert x\Vert \geq \langle p, x\rangle
\end{align*}
from which follows that $ \langle p, x\rangle=\varepsilon\Vert p\Vert_*$ which conclude the proof.
\end{proof}

\end{document}